\documentclass[12pt]{article}
\usepackage{arxiv}
\usepackage{lipsum}
\usepackage{hyperref}
\usepackage{amsmath,amsthm,amssymb}
\usepackage{setspace}
\usepackage{algorithm}
\usepackage{algorithmic}
\usepackage{tocloft}
\usepackage{graphicx}
 \usepackage{thmtools,thm-restate}
\graphicspath{ {./images/} }
\newtheorem{thm}{Theorem}
\newtheorem*{thm*}{Theorem}
\newtheorem{lem}[thm]{Lemma}
\newtheorem*{lem*}{Lemma}
\newtheorem{deff}{Definition}
\newtheorem*{deff*}{Definition}
\newtheorem{rmk}{Remark}
\newtheorem{prop}[thm]{Proposition}
\newtheorem*{prop*}{Proposition}

\newtheorem{cor}[thm]{Corollary}
\newtheorem*{cor*}{Corollary}
\newtheorem{prob*}{Problem}

\newcommand{\R}{\mathbb{R}}

\newcommand{\N}{\mathbb{N}}

\newcommand{\Norm}[1]{\left\Vert #1 \right\Vert}

\newcommand{\abs}[1]{\lvert#1 \rvert} 
\newcommand{\inn}[1]{\langle #1\rangle}

\newcommand\restr[2]{{
  \left.\kern-\nulldelimiterspace 
  #1 
  \vphantom{\big|} 
  \right|_{#2} 
  }}

\DeclareMathOperator*{\sgn}{sign}
\DeclareMathOperator*{\spn}{span}
\DeclareMathOperator*{\id}{Id}
\DeclareMathOperator*{\hing}{hinge}
\title{Representer Theorems for Metric and Preference Learning: Geometric Insights and Algorithms\thanks{A shorter version appeared in the \emph{Proc.\ Int.\ Conf.\ on Artificial Intelligence and Statistics (AISTATS 2025)} \cite{morteza2025representer}}}
\author{%
  Peyman Morteza\\
  Department of Computer Sciences\\
  University of Wisconsin-Madison\\
  \texttt{peyman@cs.wisc.edu} \\
  }
\date{}

\hypersetup{
    colorlinks=true,
    linkcolor=blue,
    filecolor=magenta,      
    urlcolor=cyan,
}
\begin{document}
\bibliographystyle{amsalpha} 
\maketitle
\begin{abstract}
We develop a mathematical framework to address a broad class of metric and preference learning problems within a Hilbert space. We obtain a novel representer theorem for the simultaneous task of metric and preference learning. Our key observation is that the representer theorem for this task can be derived by regularizing the problem with respect to the norm inherent in the task structure. For the general task of metric learning, our framework leads to a simple and self-contained representer theorem and offers new geometric insights into the derivation of representer theorems for this task. In the case of Reproducing Kernel Hilbert Spaces (RKHSs), we illustrate how our representer theorem can be used to express the solution of the learning problems in terms of finite kernel terms similar to classical representer theorems. Lastly, our representer theorem leads to a novel nonlinear algorithm for metric and preference learning. We compare our algorithm against challenging baseline methods on real-world rank inference benchmarks, where it achieves competitive performance. Notably, our approach significantly outperforms vanilla ideal point methods and surpasses strong baselines across multiple datasets. Code available at: \url{https://github.com/PeymanMorteza/Metric-Preference-Learning-RKHS}

\end{abstract}
\setcounter{tocdepth}{1}
\setlength{\cftbeforesecskip}{0.8em}
{
\begingroup
\setstretch{0.1}   
\tableofcontents
\endgroup
}
\vspace{-1mm}
\section{Introduction}
\label{sec:overview}
In machine learning, when given a set of objects or samples with only partial information, a key challenge is to infer their relationships and establish meaningful comparisons across the set. Many real-world AI applications require this capability to make informed decisions despite incomplete or limited observations. For example, large language models (LLMs) rely on ranking or scoring mechanisms to better align their generated responses with human feedback \cite{ouyang2022training,rafailov2023direct}. In computer vision, ranking is crucial in retrieval and similarity-based tasks \cite{cakir2019deep}. Similarly, recommender systems must prioritize items according to user preferences \cite{hsieh2017collaborative}.

A common approach to tackling ranking and preference learning problems involves learning a metric to quantify distances between embedded samples or identifying a reference point for proximity-based ranking \cite{furnkranz2010preference,kulis2013metric,bellet2013survey,jamieson2011active}. These models are widely applied across domains and extensively studied. However, their standard versions often struggle to capture nonlinear aspects of the problem. While theoretically and empirically well-explored, the systematic study of their kernelized counterparts has many open directions, despite the powerful framework kernel methods provide for modeling complex, nonlinear relationships.

In this paper, we address this gap by introducing a novel mathematical framework designed to investigate a broad spectrum of metric and preference learning problems within Hilbert spaces. Our focus revolves around two common tasks in this domain: the simultaneous task of metric and preference learning from pairwise comparisons \cite{jamieson2011active,massimino2021you,xu2020simultaneous,canal2022one}, and the task of metric learning from triplet comparisons \cite{ye2019fast,jain2016finite,mason2017learning}. In the simultaneous task, given a set of embedded samples $S:=\{x_{1},...,x_{m}\}\subset \mathbb{R}^{d}$ the learner is searching for an {\it ideal point} $u\in\R^{d}$ and a metric that aligns with the given partial binary responses of the form "$u$ prefers $x_{i}$ over $x_{j}$". We show that our framework leads to the first representer theorem for this task. The primary technical challenge in obtaining a representer theorem is that when the problem is lifted to a Hilbert space (potentially infinite dimensional), the ideal point, $u$, is unknown to the learner, and it may not lie on the subspace spanned by embedded samples. We define the space of {\it generalized Mahalanobis inner products} on a Hilbert space and demonstrate that the representer theorem can be naturally derived when formulated with respect to the norm induced by the inner product coming from this space. In the task of metric learning from triplet comparison, the learner aims to learn a metric that aligns with binary responses of the form "$x_{k}$ is closer to $x_{i}$ than $x_{j}$". Here, our framework yields a simple and self-contained representer theorem, offering fresh geometric insights into the derivation of such theorems. As an application, in the case of Reproducing Kernel Hilbert Spaces (RKHSs), we illustrate how our representer theorems can be used to express the solution of the original infinite-dimensional problem by solving a finite-dimensional counterpart. This leads to the development of new nonlinear algorithms specifically designed for these problems. We apply our algorithm to rank inference benchmarks and demonstrate that it is highly competitive, outperforming many strong baseline methods. The contribution of our work can be summarized as follows: 
\begin{itemize}
    \item We develop a novel mathematical framework to study a broad spectrum of metric and preference learning problems within Hilbert space. We define the notion of generalized Mahalanobis inner products on a Hilbert space, and precisely characterize their restriction to finite-dimensional subspaces as demonstrated in Theorem \ref{thm:mahal-char-rest}. 
    \item We show that our framework yields the first and novel representer theorem, Theorem \ref{thm:reper}, for the simultaneous task of metric and preference learning and also leads to a simple and self-contained representer theorem, Theorem \ref{thm:triplet_repr}, for the task of metric learning. 
    \item In the case of RKHSs, we show that our representer theorems can be used to express solutions to learning problems using a finite number of kernel terms (Proposition \ref{prop:finite-euc} and Proposition \ref{prop:finite-euc-Gram}). Furthermore, we demonstrate that this formulation leads to the development of new algorithms (Algorithm \ref{alg:ker_ideal_train} and Algorithm \ref{alg:ker_ideal_test}) for nonlinear metric and preference learning. 
    \item We present empirical evaluations of our algorithm on both synthetic and real datasets, demonstrating its competitive performance.
\end{itemize}
We close the introduction by providing an outline of this work. In Section \ref{sec:problem-euc}, we revisit the simultaneous task of metric and preference learning from pairwise comparisons and the task of metric learning from triplet comparison in finite-dimensional Euclidean spaces. In Section \ref{sec:generalized-mahal}, we define the space of generalized Mahalanobis inner products and precisely characterize how they behave when restricted to finite-dimensional subspaces. In Section \ref{sec:repr-thms}, we utilize our framework and derive our representer theorems for the simultaneous task, and for the triplet task. In Section \ref{sec:rkhs}, we investigate the case of RKHS and demonstrate how our representer theorems can be used to convert infinite-dimensional learning problems to finite-dimensional counterparts in Euclidean spaces and present a new kernelized algorithm derived from our representer theorem for metric and preference learning. In Section \ref{sec:exper}, we evaluate this algorithm on both synthetic and real-world benchmarks to analyze the effect of regularization and compare its performance with other methods.
 Finally, we conclude our work in Section \ref{sec:con} after reviewing related works in Section \ref{sec:related}. The Appendix contains detailed proofs of the main results as well as supplementary background material.
\section{Metric and Preference Learning Problem in Euclidean Spaces}
\label{sec:problem-euc}
In this section, we revisit the formulation of metric and preference learning in Euclidean spaces. Throughout this section, we work with input space $\R^{d}$ and $\mathbb{S}^{d}_{+}$ the space of $d\times d$ positive definite matrices and $\sgn(\cdot)$ denotes the sign function. Next, recall the definition of Mahalanobis distance,
\begin{deff}[\cite{mahalanobis1936generalized}]
    Given $M\in \mathbb{S}^{d}_{+}$, the Mahalanobis distance correspondes to $M$, denoted by $d_{M}$, is defined by,
    \begin{align*}
        d^{2}_{M}(x,y)=(x-y)^{\top}M(x-y),
    \end{align*}
    for $x,y\in \R^{d}$. 
\end{deff}
\begin{rmk}
    It's a standard fact that the distance $d_{M}(\cdot,\cdot)$ is induced by a norm, which in turn is induced by an inner product. The Mahalanobis distance is frequently employed in existing algorithms for metric and preference learning \cite{kulis2013metric,bellet2013survey}.
\end{rmk}
Next, we revisit the formulation of two primary tasks in finite-dimensional input space, which we intend to explore in general Hilbert spaces in subsequent sections.
\paragraph{Simultaneous Metric and Preference Learning from Pairwise Comparisons}
In the task of simultaneous metric and preference learning from paired comparisons\cite{jamieson2011active,xu2020simultaneous,canal2022one}, beginning with a set of embedded samples, $S:=\{x_{1},...,x_{m}\}\subset \mathbb{R}^{d}$, the learner aims to acquire a background preference metric $d_{M}$ associated with $M \in \mathbb{S}^{d}_{+}$ and an ideal point $u \in \mathbb{R}^{d}$, given the following data,$(z^{i},y^{i}), \text{ for } 1\le i\le n$, where $n\le \binom{m}{2}$ and $z^{i}=(z^{i}_{1},z^{i}_{2})\in S \times S$ is a pair of samples from $S$ and $y^{i}\in \{-1,+1\}$ represents if $u$ prefers $z^{i}_{1}$ over $z^{i}_{2}$ or not in the following sense, $y^{i}=\sgn(d_{M}(z_{1}^{i},u)-d_{M}(z^{i}_{2},u))$. Given the provided data, our objective is to learn the ideal point $u \in \mathbb{R}^{d}$ and the preference metric $d_{M}$ associated with $M \in \mathbb{S}^{d}_{+}$. Next, the learning problem can be formalized by considering the associated Empirical Risk Minimization (ERM) as follows, 
\begin{align}
\label{eq:simul-finite}
	\min_{A\in  \mathbb{S}^{d}_{+},u\in \R^{d}}\frac{1}{n}\sum_{i=1}^{n}\ell(d^{2}_{A}(z^{i}_{1},u)-d^{2}_{A}(z^{i}_{2},u),y^{i}),\tag{PF}
\end{align}
where $\ell$ is a general loss function. Note that, in the above formulation, we considered a general form of loss that for $1\le i\le n$, depends on $d_{A}^{2}(z^{i}_{1},u)-d^{2}_{A}(z^{i}_{2},u)$, the preference measure of $u$ between $z^{i}_{1}$ and $z^{i}_{2}$, and response value $y^{i}$ which indicates the preference sign based on the ground truth metric $d^{2}_{M}$. Our framework works for general loss as explained above. We next elaborate on two special cases of interest for theoretical analysis and practical applications. 
\begin{itemize}
    \item $\ell(d^{2}_{A}(z^{i}_{1},u)-d^{2}_{A}(z^{i}_{2},u),y^{i}):=\ell_{0/1}(\sgn(d^{2}_{A}(z^{i}_{1},u)-d^{2}_{A}(z^{i}_{2},u)),y^{i})$,
    \item $\ell(d^{2}_{A}(z^{i}_{1},u)-d^{2}_{A}(z^{i}_{2},u),y^{i}):=\ell_{conv}(\big{(}d^{2}_{A}(z^{i}_{1},u)-d^{2}_{A}(z^{i}_{2},u)\big{)}\cdot y^{i})$.
\end{itemize}
In the first item, $\ell_{0/1}(\cdot)$ denotes the $0/1$ loss, which is beneficial for theoretical analysis, such as examining sample complexity. In the second item, $\ell_{conv}$ represents any convex Lipschitz loss (e.g., Hinge loss), making it practical for applications as one can study its convex relaxation, as demonstrated in works such as \cite{xu2020simultaneous,canal2022one}. 
\begin{rmk}
While Problem \ref{eq:simul-finite} has been extensively studied, there is currently no kernelization framework for the simultaneous task. The main challenge in developing such a framework is the existence of the ideal point $u$. In fact, when the problem is lifted to a Hilbert space (potentially infinite-dimensional), there is no guarantee that the ideal point $u$ lies on the subspace spanned by embedded samples, as is the case in the classical setting for other learning problems. We show how to address this by equipping the ambient space with a natural inner product. 
\end{rmk}
\paragraph{Learning Metrics from Triplet Comparisons}
In the task of metric learning from triplet comparisons\cite{ye2019fast,jain2016finite,mason2017learning}, beginning with a set of embedded samples, $S:=\{x_{1},...,x_{m}\}\subset \mathbb{R}^{d}$, the learner aims to acquire a background preference metric $d_{M}$ associated with $M \in \mathbb{S}^{d}_{+}$, given the following data,
$(z^{i},y^{i}), \text{ for } 1\le i\le n$, where $n\le \binom{m}{3}$ and $z^{i}=(z^{i}_{1},z^{i}_{2},z^{i}_{3})\in S\times S\times S$ is a triplet sample and $y^{i}\in \{+1,-1\}$ represents if $z^{i}_{1}$ is similar to $z^{i}_{2}$ or $z^{i}_{3}$ in the following sense,$y^{i}=\sgn(d_{M}(z^{i}_{1},z^{i}_{2})-d_{M}(z^{i}_{1},z^{i}_{3}))$. Given the provided data, our objective is to learn the metric $d_{M}$ associated with $M \in \mathbb{S}^{d}_{+}$. Next, the learning problem can be formalized by considering the associated Empirical Risk Minimization (ERM) as follows,
\begin{align}
\label{eq:triplet-finite}
	\min_{A\in \mathbb{S}^{d}_{+}}\frac{1}{n}\sum_{i=1}^{n}\ell(d^{2}_{A}(z^{i}_{1},z^{i}_{2})-d^{2}_{A}(z^{i}_{1},z^{i}_{3}),y^{i}).\tag{TF}
\end{align}
where $\ell$ is a general loss function. Similar to the simultaneous task, \ref{eq:simul-finite}, we can consider the following two special cases of interest for the loss. 
\begin{itemize}

    \item  $ 
    \ell(d^{2}_{A}(z^{i}_{1},z^{i}_{2})-d^{2}_{A}(z^{i}_{2},z^{i}_{3}),y^{i}):=\ell_{0/1}(\sgn(d^{2}_{A}(z^{i}_{1},z^{i}_{2})-d^{2}_{A}(z^{i}_{2},z^{i}_{3})),y^{i})$,
    \item  $\ell(d^{2}_{A}(z^{i}_{1},z^{i}_{2})-d^{2}_{A}(z^{i}_{2},z^{i}_{3}),y^{i}):=\ell_{conv}(\big{(}d^{2}_{A}(z^{i}_{1},z^{i}_{2})-d^{2}_{A}(z^{i}_{2},z^{i}_{3})\big{)}\cdot y^{i})$.
\end{itemize}
In the first item, $\ell_{0/1}(\cdot)$ denotes the $0/1$ loss, which is beneficial for theoretical analysis, such as examining sample complexity. In the second item, $\ell_{conv}$ represents any convex Lipschitz loss (e.g., Hinge loss). 
\begin{rmk}
Note that in the triplet setting, the ideal point is absent. Prior works have explored kernelization frameworks for Mahalanobis metric learning \cite{chatpatanasiri2010new,kulis2013metric,jain2012metric}. The framework developed in this paper yields a straightforward and self-contained representer theorem for the triplet setting. Importantly, it does not rely on the Kernel PCA trick and offer new geometric insights.
\end{rmk}
\section{Space of Generalized Mahalanobis Inner Products}
\label{sec:generalized-mahal}
In this section, we develop a general framework that can be utilized to investigate a wide range of metric and preference learning problems, including \ref{eq:simul-finite} and \ref{eq:triplet-finite}, within a general Hilbert space. We establish a couple of technical results that we later leverage to derive our representer theorems. Specifically, we introduce the space of generalized Mahalanobis inner products (see Definition \ref{deff:gen-mahal}) and precisely characterize their restriction to finite-dimensional subspaces, as stated in Theorem \ref{thm:mahal-char-rest}. Throughout this section, we assume $\mathcal{H}$ is a Hilbert space equipped with inner product $\inn{\cdot,\cdot}_{\mathcal{H}}$ and associated norm $\Norm{\cdot}_{\mathcal{H}}$ (We review some standard facts and definitions regarding linear operators on Hilbert spaces in the Appendix). To begin, we require a class of Mahalanobis distances in a general Hilbert space. We define the following,

\begin{deff}[Space of generalized Mahalanobis inner products]
\label{deff:gen-mahal}
Space of generalized Mahalanobis inner product on a Hilbert space $\mathcal{H}$ is defined to be the following set,
\begin{align*}
	\mathcal{F}_{\mathcal{H}} := 
	\left\{ A: \mathcal{H} \to \mathcal{H} \mid 
	\begin{aligned}
		& A \text{ is bounded,}
		& \text{strictly positive, self-adjoint}
	\end{aligned}
	\right\}.
\end{align*}
\end{deff}
\begin{rmk}
We can view $\mathcal{F}_{\mathcal{H}}$ as the set that parametrizes the {\it generalized Mahalanobis inner product}, as we explain next. For $A\in \mathcal{F}_{\mathcal{H}}$, we can consider the associated inner product defined by, $\inn{x,y}_{A}:=\inn{Ax,y}_{\mathcal{H}}$. It can be seen that for $A\in \mathcal{F}_{\mathcal{H}}$, $\inn{\cdot,\cdot}_{A}$ defines an inner product on $\mathcal{H}$ (See Proposition \ref{prop:inner_equiv} below). Moreover, when $\mathcal{H}$ is finite-dimensional and an orthonormal basis is chosen with respect to $\inn{\cdot,\cdot}_{\mathcal{H}}$, $A$ can be represented by a positive definite matrix $M$ and $\Norm{\cdot}_{A}$ coincides with the Mahalanobis norm corresponding to $M$ (See Lemma \ref{lem:matrix-repr} below). In this sense, elements in $\mathcal{F}_{\mathcal{H}}$ can be viewed as an infinite-dimensional analogue of the standard Mahalanobis norms \cite{mahalanobis1936generalized} in Euclidean spaces. 
   In other words, we can regard elements of $\mathcal{F}_{\mathcal{H}}$ as a generalized version of symmetric, positive-definite matrices in finite-dimensional space. 
\end{rmk}
\begin{rmk}
Previous studies \cite{chatpatanasiri2010new, kulis2013metric, jain2012metric} have explored the learning of a general operator $L$ without additional assumptions and have focused on the distance of the image (e.g., $d_{L}(x,y):=\Norm{Lx-Ly}_{\mathcal{H}}$) as the metric corresponding to $L$. This approach, while capable of kernelizing \eqref{eq:triplet-finite}, fails to work for the kernelization of \eqref{eq:simul-finite}. Utilizing $\mathcal{F}_{\mathcal{H}}$ offers an important advantage, as its elements are uniquely associated with inner products on $\mathcal{H}$. Leveraging these inner products enables us to develop a novel regularizer that plays a crucial role in formulating our representer theorems, as we will elaborate in Section \ref{sec:repr-thms}. 
\end{rmk}
The next proposition shows that associated with any element $\mathcal{F}_{\mathcal{H}}$ comes with a natural inner product. 
\begin{prop}
\label{prop:inner_equiv}
Let $A\in \mathcal{F}_{\mathcal{H}}$, then,
\begin{align*}
	\inn{x,y}_{A}:=\inn{Ax,y}_{\mathcal{H}},
\end{align*}
defines an inner product on $\mathcal{H}$ that is equivalent to $\inn{\cdot,\cdot}_{\mathcal{H}}$. Conversely, let $g(\cdot,\cdot)$ be an inner product on $\mathcal{H}$ equivalent to $\inn{\cdot,\cdot}_{\mathcal{H}}$ then there exist a unique $A\in \mathcal{F}_{\mathcal{H}}$ such that $g=\inn{\cdot,\cdot}_{A}$.
\end{prop}
\begin{proof}
Proof can be found in the appendix.
\end{proof}
\begin{rmk}
The proposition \ref{prop:inner_equiv} establishes a one-to-one correspondence between elements of $\mathcal{F}_{\mathcal{H}}$ and their associated inner products. Consequently, rather than dealing directly with operators, one can operate with these inner products. This eliminates the necessity of referencing a specific basis for deriving represented theorems. Such an approach is particularly useful in general Hilbert spaces (e.g., RKHSs), where a canonical basis is not provided, in contrast to $\mathbb{R}^{n}$. 
\end{rmk}
The subsequent lemma demonstrates that when $\mathcal{H}$ is finite-dimensional with a predetermined orthonormal basis, the elements of $\mathcal{F}_{\mathcal{H}}$ align with the standard Mahalanobis norm, corresponding to the representation of the operator with respect to the basis. 
\begin{lem}
\label{lem:matrix-repr}
    Let $\mathcal{H}$ be a $n$-dimensional Hilbert space and $\{e_{1},...,e_{n}\}$ be an orthonormal basis for $\mathcal{H}$. Let $M$ denote the representation of $A$ with respect to this basis. Then the inner product associated with the Mahalanobis norm coming from $M$ coincides with $\inn{\cdot,\cdot}_{A}$. 
\end{lem}
\begin{proof}
Proof can be found in the appendix.
\end{proof}
\begin{rmk}
The  lemma implies that for any orthonormal basis, the Mahalanobis inner product of the corresponding matrix aligns with $\inn{\cdot,\cdot}_{A}$. The lemma can be used to perform computations independent of any basis, thereby simplifying the computational process. 
\end{rmk}
Next, considering a finite-dimensional subspace $V \subset \mathcal{H}$, $V$ can be viewed as a Hilbert space inheriting its inner product from $\mathcal{H}$. We can define $\mathcal{F}_{V}$ in the same manner as outlined in Definition \ref{deff:gen-mahal}. The following definition delineates which elements in $\mathcal{F}_{V}$ can be paired with elements in $\mathcal{F}_{\mathcal{H}}$.
\begin{deff}
\label{def:extension}
Let $\mathcal{H}$ be a Hilbert space and $V\subset \mathcal{H}$ be a finite-dimensional subspace. Let $A\in \mathcal{F}_{\mathcal{H}}$ and $B \in \mathcal{F}_{V}$. We say $A$ is {\it Mahalanobis extension} of $B$ (or $B$ is a {\it Mahalanobis restrict} of $A$) if,
\begin{align*}
   \restr{\inn{\cdot,\cdot}_{A}}{V}=\inn{\cdot,\cdot}_{B}.
\end{align*}
\end{deff}
We close this section by stating the following theorem which completely characterizes when $A\in \mathcal{F}_{\mathcal{H}}$ on ambient space is the Mahalanobis extension of $B\in \mathcal{F}_{V}$. This will be used in the next section to obtain our representer theorems.
\begin{thm}
\label{thm:mahal-char-rest}
$A\in \mathcal{F}_{\mathcal{H}}$ is Mahalanobis extension of $B\in \mathcal{F}_{V}$ iff,
\begin{align*}
    B=PA|_{V},
\end{align*}
where,
\begin{align*}
    P:\mathcal{H}\to V\subset \mathcal{H},
\end{align*}
is projection operator with respect to $\inn{\cdot,\cdot}_{\mathcal{H}}$. Moreover, each $B\in \mathcal{F}_{V}$ has at least one Mahalanobis extension.
\end{thm}
\begin{proof}
Proof can be found in the appendix.
\end{proof}
\section{Representer Theorems for Metric and Preference Learning}
\label{sec:repr-thms}
In this section, we leverage our framework to delve into metric and preference learning problems in Hilbert space. Specifically, we establish a new representer theorem for \ref{eq:simul-finite} and present a simple and self-contained representer theorem for \ref{eq:triplet-finite}. We begin by formulating \ref{eq:simul-finite} and \ref{eq:triplet-finite} in a general Hilbert space $\mathcal{H}$, using the framework developed in Section \ref{sec:generalized-mahal}.
\paragraph{Metric and Preference Learning Problem in Hilbert Spaces}
Here we reformulate \ref{eq:simul-finite} and \ref{eq:triplet-finite} using the framework developed in Section \ref{sec:generalized-mahal} in general Hilbert spaces. In the task of simultaneous metric and preference learning from paired comparisons, beginning with a set of embedded samples, $S:=\{x_{1},...,x_{m}\}\subset \mathcal{H}$, we aim to learn a background preference metric associated to $A \in \mathcal{F}_{\mathcal{H}}$ and an ideal point $u \in \mathcal{H}$, given the following data,$(z^{i},y^{i}), \text{ for } 1\le i\le n$,
where $n\le \binom{m}{2}$ and $z^{i}=(z^{i}_{1},z^{i}_{2})\in S \times S$ is a pair of samples from $S$ and $y^{i}\in \{-1,+1\}$ is defined by, $y^{i}=\sgn(\Norm{z_{1}^{i}-u}_{A}^{2}-\Norm{z^{i}_{2}-u}_{A}^{2})$. Given the provided data, the learning problem can be formalized by considering the associated Empirical Risk Minimization (ERM) as follows, 
\begin{align}
\label{eq:simul-infinite}
	\min_{A\in  \mathcal{F}_{\mathcal{H}},u\in \R^{d}}\frac{1}{n}\sum_{i=1}^{n}\ell(\Norm{z^{i}_{1}-u}^{2}_{A}-\Norm{z^{i}_{2}-u}^{2}_{A},y^{i}),\tag{PI}
\end{align}
In the task of metric learning from triplet comparisons, we aim to acquire a background preference metric $\Norm{\cdot}_{A}$ associated with $A \in \mathcal{F}_{\mathcal{H}}$, given the following data, $(z^{i},y^{i}), \text{ for } 1\le i\le n$, where $n\le \binom{m}{3}$ and $z^{i}=(z^{i}_{1},z^{i}_{2},z^{i}_{3})\in S\times S\times S$ is a triplet sample and $y^{i}\in \{+1,-1\}$ is defined by,$y^{i}=\sgn(\Norm{z^{i}_{1}-z^{i}_{2}}_{A}^{2}-\Norm{z^{i}_{1}-z^{i}_{3}}_{A}^{2})$. The learning problem can be formalized by considering the associated Empirical Risk Minimization (ERM) as follows,
\begin{align}
\label{eq:triplet-infinite}
	\min_{A\in \mathcal{F}_{\mathcal{H}}}\frac{1}{n}\sum_{i=1}^{n}\ell(\Norm{(z^{i}_{1}-z^{i}_{2}}_{A}^{2}-\Norm{z^{i}_{1}-z^{i}_{3}}_{A}^{2},y^{i}).\tag{TI}
\end{align}
where $\ell$ is a general loss function. We refer to Section \ref{sec:problem-euc} for detailed discussion on the loss function and other details for these models.
\paragraph{A Representer Theorem for Simultaneous Metric and Preference Learning} 
Note that the search space, $\mathcal{F}_{\mathcal{H}} \times \mathcal{H}$, are infinite dimensional when $\mathcal{H}$ has infinite dimension. We would like to convert the infinite-dimensional problem represented by \ref{eq:simul-infinite} into a finite-dimensional equivalent. We demonstrate that a solution to the problem \ref{eq:simul-infinite} can be obtained by solving a finite-dimensional equivalent. Additionally, we show that regularizing the problem with the appropriate norm associated with elements in $\mathcal{F}_{\mathcal{H}}$ allows the entire process to be viewed as a Representer Theorem. First, set $V:=\spn\{x_{1},...,x_{m}\}$ and consider the following finite dimensional problem,
\begin{align}
\label{eq:simul-finite-hilbert}
	\min_{A\in \mathcal{F}_{V},u\in V}\frac{1}{n}\sum_{i=1}^{n}\ell(\Norm{z^{i}_{1}-u}^{2}_{A}-\Norm{z^{i}_{2}-u}^{2}_{A},y^{i})\tag{PFH},
\end{align}

Note that $V$ is finite dimensional Hilbert space and it inherits an inner product from $\mathcal{H}$ which we denote it by $\inn{\cdot,\cdot}_{\mathcal{H}}$ again. $\mathcal{F}_{V}$ is defined in a same manner as $\mathcal{F}_{\mathcal{H}}$ but the key point here is that $\mathcal{F}_{V}$ is now finite dimensional space. We would like to know how the solutions of \ref{eq:simul-finite-hilbert} and \ref{eq:simul-infinite} are related. The following observation relates the solutions of \ref{eq:simul-finite-hilbert} to the solutions of \ref{eq:simul-infinite},
\begin{prop}
\label{prop:induced-1}
Let $A^{\ast}\in \mathcal{F}_{\mathcal{H}},u^{\ast}\in \mathcal{\mathcal{H}}$ be a solution to \ref{eq:simul-infinite}. Equip $\mathcal{H}$ with $\inn{\cdot,\cdot}_{A^{\ast}}$ and let $u^{\ast}=u^{\perp}+u^{T}$ where $u^{\perp}\in V^{\perp}$ and $u^{T} \in V$. Note that the orthogonal decomposition is with respect to $\inn{\cdot,\cdot}_{A^{\ast}}$ and not $\inn{\cdot,\cdot}_{\mathcal{H}}$. Then $(A^{\ast},u^{T})$ is also a solution to \ref{eq:simul-infinite}. Moreover, let $B^{\ast}\in \mathcal{F}_{V}$ be the Mahalanobis restrict (See Definition \ref{def:extension}) of $A^{\ast}$, then $(B^{\ast},u^{T})$ is a solution to \ref{eq:simul-finite-hilbert} with same optimal value. Conversely, let $(B^{\ast},u)$ be a solution to \ref{eq:simul-finite-hilbert}, and let $A^{\ast}\in \mathcal{F}_{\mathcal{H}}$ be any Mahalanobis extension of $B^{\ast}$ (See Definition \ref{def:extension}), then $(A^{\ast},u)$ is also a solution to \ref{eq:simul-infinite} with same optimal value. 
\end{prop}
\begin{proof}
	See Section \ref{sec:proofs}.
\end{proof}
Finally, Proposition \ref{prop:induced-1}, suggests that if one considers the regularized problem with right norm then we can view it as a representer theorem.
\begin{thm}[Representer Theorem for Simultaneous Metric and Preference Learning]
\label{thm:reper}
Let $\lambda>0$ and consider the following infinite dimensional regularized problem,
\begin{small}
\begin{align}
\label{eq:infinite-representer}
	\min_{A\in \mathcal{F}_{\mathcal{H}},u\in \mathcal{H}}\frac{1}{n}\sum_{i=1}^{n}\ell(\Norm{z^{i}_{1}-u}^{2}_{A}-\Norm{z^{i}_{2}-u}^{2}_{A},y^{i})+\lambda \Norm{u}^{2}_{A}, \tag{R}
\end{align}
\end{small}
and it's finite dimensional equivalent as follows,
\begin{small}
\begin{align}
\label{eq:finite-representer}
	\min_{A\in \mathcal{F}_{V},u\in V}\frac{1}{n}\sum_{i=1}^{n}\ell(\Norm{z^{i}_{1}-u}^{2}_{A}-\Norm{z^{i}_{2}-u}^{2}_{A},y^{i})+\lambda \Norm{u}^{2}_{A}\tag{F},
\end{align}
\end{small}

Let $A^{\ast}\in \mathcal{F}_{H}$ be any Mahalanobis extension of $B^{\ast}\in \mathcal{F}_{V}$. Then $(A^{\ast},u)$ is a solution to  \ref{eq:infinite-representer} iff $(B^{\ast},u)$ is a solution to \ref{eq:finite-representer} with same optimal value. In particular this implies that $u\in V$. 
\end{thm}
 \begin{proof}
See Section \ref{sec:proofs}.
 \end{proof}
\begin{rmk}
Theorem \ref{thm:reper} states that by solving a finite dimensional counterpart \ref{eq:finite-representer}, we can find the optimal value of \ref{eq:infinite-representer}. If $\mathcal{H}$ is an RKHS associated with a kernel function $k$, then we can demonstrate that the solutions of \ref{eq:finite-representer} can be expressed in terms of kernel terms (as shown in Proposition \ref{prop:finite-euc}), which resemble classical representer theorems.
\end{rmk}
\begin{rmk}
Note that Theorem \ref{thm:reper} has a regularization term $\Norm{u}^{2}_{A} $ stemming directly from equipping $\mathcal{H}$ with the generalized Mahalanobis inner product discussed in Section \ref{sec:generalized-mahal}. In many classical representer theorems, the regularization term arises directly from the original norm $\Norm{\cdot}^{2}_{\mathcal{H}}$. However, this classical approach fails to yield a representer theorem for simultaneous tasks. We illustrate this with an example as illustrated in Figure \ref{fig:difff}.    
\end{rmk}
\paragraph{An Illustrative Example}
\begin{figure*}
\begin{center}
		\includegraphics[width=0.97 \textwidth]{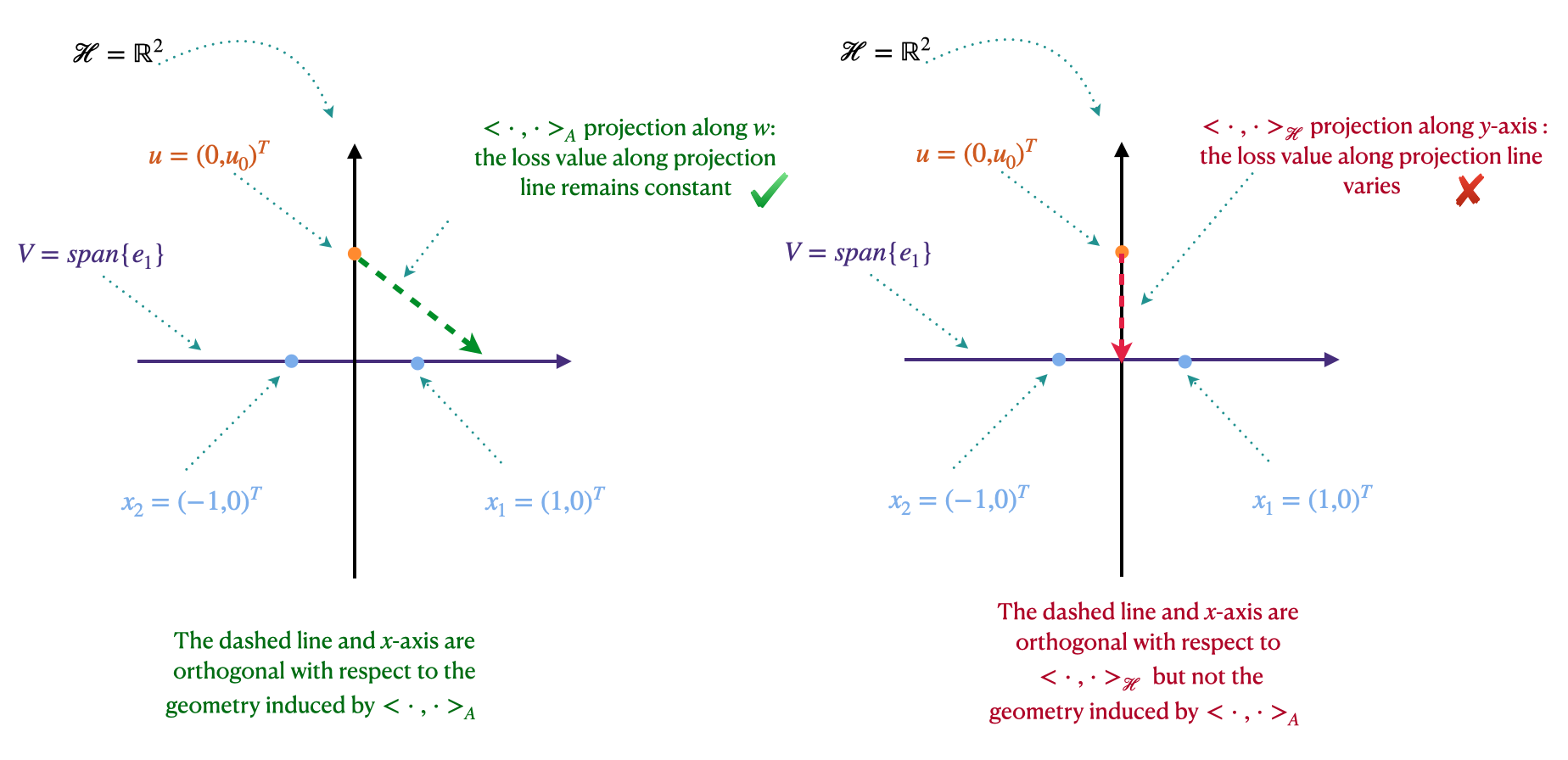}
\end{center}
\caption{Illustration of the variation of preference loss with respect to the underlying geometry. When projecting along the $\inn{\cdot,\cdot}_{\mathcal{H}}$, the loss value changes along the projection line. However, when projecting along the line induced by $\inn{\cdot,\cdot}_{A}$, the loss remains constant.}  
\label{fig:difff}
\end{figure*}
The choice of the inner product on $\mathcal{H}$ play a crucial role in the validity of Theorem \ref{thm:reper}. Consider the case $\mathcal{H}=\R^{2}$ with $S=\{x_{1},x_{2}\}$ for $x_{1}=e_{1}=(1,0)^{\top}$ and $x_{2}=-e_{1}=(-1,0)^{\top}$.
In this case, $\mathcal{H}=\R^{2},V=\spn\{e_{1}\}$. Next, assume that $u=(0,u_{0})^{\top}$ lies on the $y$-axis. This is illustrated in Figure \ref{fig:difff} . Let $A\in \mathcal{F}_{\mathcal{H}}$ has the following representation, $A=
\begin{pmatrix}
 1 &1\\
 1 & 2  
\end{pmatrix},
$
and denotes the metric on $\mathcal{H}$ that we would like to learn.  We can think of the distance difference between $\Norm{x_{1}-u}_{A}^{2}-\Norm{x_{2}-u}_{A}^{2}$ as preference loss as its sign represents whether $u$ prefers product $x_{1}$ or $x_{2}$. We would ideally like to find an equivalent problem on $V$. Following the approach in classical representer theorems, one uses the orthogonal projection of $u$ using $\inn{\cdot,\cdot}_{\mathcal{H}}$  n the $y$ axis (i.e. Euclidean projection). One can easily show that the loss depends on the location of $u$ on the projection line and this naive projection does not work. Instead, we suggest projecting along the orthogonal line that is induced by the inner product obtained from $A\in \mathcal{F}_{\mathcal{H}}$ which may not necessarily be aligned with $\inn{\cdot,\cdot}_{\mathcal{H}}$. One can easily check that for $w=(1,-1)^{\top}$,$\inn{e_{1},w}_{A}=0$, which means that the orthogonal projection with respect to geometry that is induced by $\inn{\cdot,\cdot}_{A}$ is in fact in the direction of $w$. This simple insight serves as the foundation for our representer theorem, \ref{thm:reper}, as demonstrated in Figure \ref{fig:difff}.
\paragraph{A Representer Theorem for Metric Learning from Triplet Comparison}
Next, we demonstrate how the framework we have established leads to a straightforward and self-contained representer theorem for the triplet learning task.
\begin{thm}[Representer Theorem for the Triplet Task]
\label{thm:triplet_repr}
Let $A^{\ast}\in \mathcal{F}_{\mathcal{H}}$. Then $A^{\ast}$ is a solution to 
\begin{align*}
\label{eq:RT}
\tag{RT}
	\min_{A\in \mathcal{F}_{\mathcal{H}}}\frac{1}{n}\sum_{i=1}^{n}\ell(\Norm{z^{i}_{1}-z^{i}_{2}}_{A}-\Norm{z^{i}_{1}-z^{i}_{3}}_{A},y^{i}),
\end{align*}
iff $B^{\ast}=PA^{\ast}|_{V}$, is a solution to,
\begin{align*}
\label{eq:FT}
\tag{FT}
	\min_{B\in \mathcal{F}_{V}}\frac{1}{n}\sum_{i=1}^{n}\ell(\Norm{(z^{i}_{1}-z^{i}_{2}}_{B}-\Norm{z^{i}_{1}-z^{i}_{3}}_{B},y^{i}).
\end{align*}
where,
\begin{align*}
    P:\mathcal{H}\to V\subset \mathcal{H},
\end{align*}
is projection operator with respect to $\inn{\cdot,\cdot}_{\mathcal{H}}$.
And both solutions have the same optimal value.
\end{thm}
\begin{proof}
The proof immediately follows from Theorem \ref{thm:mahal-char-rest}. Intuitively speaking, the solution of \ref{eq:RT} depends solely on the distance induced by $A$ on the subspace $V$ (given there is no ideal point in this task), which is precisely characterized by Theorem \ref{thm:mahal-char-rest}.
\end{proof}

\label{subsec:rep_trip} 

\begin{rmk}
\label{rmk:comparisons}
Note that the ideal point $u$ is absent in the triplet setting, and a representer theorem for metric learning from triplet comparisons can also be derived directly from the framework presented in \cite{chatpatanasiri2010new} which also forms the basis for \cite{tatli2025metric}. Our approach, by contrast, leads to a simple, intuitive, and self-contained representer theorem for \ref{eq:triplet-finite} and has the advantage of not relying on a specific basis for $V$ (e.g., the KPCA trick as in \cite{chatpatanasiri2010new}). The learner can use their favorite orthonormal basis to transform \ref{eq:FT} into finite-dimensional Euclidean spaces. We illustrate one possible choice for this process in the next section.

\end{rmk}

\section{Kernelized Algorithms for Metric and Preference Learning}
\label{sec:rkhs}
\begin{figure*}[t]
        \begin{algorithm}[H]
            \begin{algorithmic}[1]
                \STATE \textbf{Input:} 
                \STATE \hspace{0.5cm} Sample set $S=\{s_{1},...,s_{m}\}\subset \R^{d}$
                \STATE \hspace{0.5cm} Dataset $\mathcal{D}=\{(z^{i},y^{i}) \mid z^{i}=(z^{i}_{1},z^{i}_{2})\in S \times S, 1\le i\le n\}$
                \STATE \hspace{0.5cm} Kernel function $k(\cdot,\cdot)$
                \STATE \hspace{0.5cm} Regularization parameter $\lambda>0$
                \STATE Solve \ref{eq:finite-representer-prop} from Proposition \ref{prop:finite-euc} using numerical optimization to obtain:
                \STATE \hspace{0.5cm} $A^{\ast}\in \mathbb{S}_{m}^{+}$ (learned metric)
                \STATE \hspace{0.5cm} $u^{\ast}\in \R^{m}$ (preference/ideal vector)
                \STATE \textbf{Output:} Trained Ideal Point model:
                \STATE \hspace{0.5cm} Learned parameters: $A^{\ast}\in \mathbb{S}_{m}^{+}$, $u^{\ast}\in \R^{m}$
                \STATE \hspace{0.5cm} Use these parameters to predict preferences for test sample pairs.
            \end{algorithmic}
            \caption{Kernelized Ideal Point Algorithm: Training}
            \label{alg:ker_ideal_train}
        \end{algorithm}
\end{figure*}
\begin{figure*}[t]
       \begin{algorithm}[H]
            \begin{algorithmic}[1]
                \STATE \textbf{Input:} 
                \STATE \hspace{0.5cm} Test sample pair $z_{\text{test}} = (z_{\text{test},1}, z_{\text{test},2}) \in S \times S$
                \STATE \hspace{0.5cm} Trained model parameters: $A^{\ast} \in \mathbb{S}_{m}^{+}, u^{\ast} \in \R^{m}$
                \STATE Compute the $\alpha$ representation using Proposition \ref{prop:finite-euc}:
                \STATE \hspace{0.5cm} $\alpha_{z_{\text{test},1}}, \alpha_{z_{\text{test},2}}\in \R^{m}$
                \STATE Compute the preference score using Proposition \ref{prop:finite-euc}:
                \STATE \hspace{0.5cm} $s = \| \alpha_{z_{\text{test},1}} - u^{\ast} \|^{2}_{A^{\ast}} - \| \alpha_{z_{\text{test},2}} - u^{\ast} \|^{2}_{A^{\ast}}$
                \STATE \textbf{Output:} Preference prediction based on $\sgn(s)$
            \end{algorithmic}
            \caption{Kernelized Ideal Point Algorithm: Testing}
            \label{alg:ker_ideal_test}
        \end{algorithm}
\end{figure*}
In Section \ref{sec:repr-thms}, we established representer theorems for metric and preference learning in a general Hilbert space. Now, we turn our attention to the case of reproducing kernel Hilbert spaces (RKHS). Specifically, we aim to apply Theorem \ref{thm:reper} and Theorem \ref{thm:triplet_repr} in the context of RKHS so that we formulate problems in Euclidean spaces for practical application. To this end, we consider the setting where $\mathcal{X}=\R^{d}$ and $k$ is a kernel on $\mathcal{X}$. Let $\mathcal{H}_{k}$ denote the RKHS associated with $k$. Next, we assume that $S^{\prime}=\{s_{1},...,s_{m}\}\subset \mathcal{X}$ is a set of embedded samples in $\mathcal{X}$. Next, using $k$, for $1\le i\le m$, we can associate the following in $\mathcal{H}_{k}$, $x_{i}:=k(s_{i},\cdot)$, and therefore we can transform the problem into $\mathcal{H}_{k}$ as in Section \ref{sec:repr-thms}. Now Theorem \ref{thm:reper} and Theorem \ref{thm:triplet_repr} apply to convert infinite dimensional optimization to finite-dimensional equivalent. The following proposition finds an equivalent optimization in finite-dimensional Euclidean space.  
\begin{prop}
\label{prop:finite-euc}
Assume $S=\{k(\cdot,s_{1}),...,k(\cdot,s_{m})\}$, forms a linearly independent set. Then:
\begin{itemize}
\item 
Solutions of \ref{eq:finite-representer} and \ref{eq:FT} can be represented in terms of kernel terms as follows,
\begin{align*}
	&A=\sum_{i,j=1}^{m}a_{ij}k(\cdot,s_{i})\otimes k(\cdot,s_{j}),\\
	&u=\sum_{i=1}^{m}b_{i}k(\cdot,s_{i}),
\end{align*}
where in above we identified the space of linear maps on $V$ with $V\otimes V$ using $\inn{\cdot,\cdot}_{\mathcal{H}}$. 
\item 
Moreover, by choosing an orthonormal basis for $V$, there is a one-to-one correspondence between solutions of \ref{eq:finite-representer} with the solutions of the following problem in $\R^{m}$. 	
\begin{align}
\label{eq:finite-representer-prop}
	\min_{A\in \mathbb{S}^{m}_{+}, u\in \R^{m}} 
	\sum_{i=1}^{n} \ell\Big( 
	& \Norm{\alpha_{z^{i}_{1}}-u}^{2}_{A} - \Norm{\alpha_{z^{i}_{2}}-u}^{2}_{A}, y^{i} \Big)  \nonumber + \lambda \Norm{u}_{A}^{2} \tag{EF}
\end{align}

and one to one correspondence between the solutions of \ref{eq:FT} with the solutions of the following problem in $\R^{m}$,
\begin{align}
\label{eq:finite-representer-triplet}
	\min_{A\in \mathbb{S}^{m}_{+}}\sum_{i=1}^{n}\ell(\Norm{\alpha_{z^{i}_{1}}-\alpha_{z^{i}_{2}}}^{2}_{A}-\Norm{\alpha_{z^{i}_{1}}-\alpha_{z^{i}_{3}}}^{2}_{A},y^{i})\tag{EFT},
\end{align}
where $\mathbb{S}^{m}_{+}$ denotes space of symmetric positive definite matrices on $\R^{m}$ and for $u\in \R^{m}$ and $A\in\mathbb{S}^{m}_{+}$,
\begin{align*}
	\Norm{u}^{2}_{A}=u^{\top}Au,
\end{align*}
 and for $x=k(\cdot,s)\in \mathcal{H}$, $\alpha_{x}=(\alpha_{1},...,\alpha_{m})^{\top}\in\R^{m}$ is the representation of $x$ with respect to the orthonormal basis obtained by Gram-Schmidt process on elements of $S$ and each $\alpha_{i}$ is defined by,
\begin{align*}
	\alpha_{i}=\frac{1}{\sqrt{D_{i-1}D_{i}}}\begin{vmatrix}
k(s_{1},s_{1}) & \dots & \dots & k(s_{1},s_{i}) \\ 
\vdots & \ddots & \ddots & \vdots \\ 
k(s_{i-1},s_{1}) & \dots & \dots & k(s_{i-1},s_{i}) \\ 
k(s,s_{1}) & \dots & \dots & k(s,s_{i})  \notag
\end{vmatrix},
\end{align*}
and $D_{0}=1$ and for $1\le i\le m$, $D_{i}$ is defined as follows,
\begin{align*}
	D_{i}=\begin{vmatrix}
k(s_{1},s_{1}) & \dots & \dots & k(s_{1},s_{i}) \\ 
\vdots & \ddots & \ddots & \vdots \\ 
k(s_{i-1},s_{1}) & \dots & \dots & k(s_{i-1},s_{i}) \\ 
k(s_{i},s_{1}) & \dots & \dots & k(s_{i},s_{i})  \notag
\end{vmatrix}.
\end{align*}
\end{itemize}
\end{prop}
\begin{proof}
	The first item follows from the fact that elements of $S$ span $V$. The proof of the second item is a direct application of  Lemma \ref{lem:reper_euc-app} relies on the determinant form of the Gram-Schmidt process and Leibniz determinant formula which we recall in the appendix. 
\end{proof}
\begin{rmk}
\label{rmk:comp_gram_iter}
In Proposition \ref{prop:finite-euc}, we used the determinant form of the Gram-Schmidt process. Similarly, the computation of $\alpha_{x}$ in Proposition \ref{prop:finite-euc} can be efficiently performed using its iterative form. The assumption of linear independence for $S$ in Proposition \ref{prop:finite-euc} is mainly for clarity. If $S$ is not linearly independent, a spanning subset can be selected using the iterative Gram-Schmidt process.

\end{rmk}
\subsection{Eliminating Basis Dependence via the Gram Matrix}
\label{subsec:Gram}
Proposition \ref{prop:finite-euc} relies on constructing an orthonormal basis for $V$. 
While this can be done using the Gram--Schmidt procedure, as discussed in the previous subsection, 
it is often more desirable to express the learning process in a compact form purely in terms of the data 
(e.g., the Gram matrix or kernel products of the data). 
Such formulations make the learning process more interpretable, in line with classical kernel methods and representer theorems. 
Moreover, in certain regimes, having a compact representation offers computational advantages, 
since Gram--Schmidt can be expensive or unnecessary.  

The main challenge in obtaining such a compact formulation lies in parameterizing PSD matrices 
on the subspace $V \subset \mathcal{H}_{k}$, which in turn required an orthonormal basis and hence the Gram--Schmidt step. 
We next show how this requirement can be avoided. 
The following lemma establishes that bilinear forms (and in particular PSD matrices) 
can be expressed directly in terms of the Gram matrix of the data points generating the subspace.

\begin{lem}
\label{lem:gram-lemma}
Let \( (V,\inn{\cdot,\cdot}_{V}) \) be a real inner product space, and let \( \{v_1, \dots, v_m\} \subset V \) be a (possibly linearly dependent) set of vectors. Let \( G \in \mathbb{R}^{m \times m} \) denote the Gram matrix:
\[
G_{ij} := \langle v_i, v_j \rangle_V.
\]
Let \( A \in \mathbb{R}^{m \times m} \), and define the bilinear form \( \widetilde{A} \in V \otimes V \) as:
\[
\widetilde{A} := \sum_{i,j=1}^m a_{ij} \, v_i \otimes v_j.
\]
Then for any \( u=(u_1,...,u_m)^{\top} \in \mathbb{R}^m \), the vector \( \bar{u} := \sum_{k=1}^m u_k v_k \in \operatorname{span}\{v_i\} \) satisfies:
\[
\langle \widetilde{A}, \bar{u} \otimes \bar{u} \rangle = u^\top G A G u.
\]

\end{lem}
\begin{proof}
	See Section \ref{sec:proofs}.
\end{proof}
A corollary of Lemma \ref{lem:gram-lemma} is that a bilinear form $\tilde{A}$ is positive semidefinite if and only if $G A G$ is positive semidefinite, as formally stated below.
\begin{cor}
\label{cor:bilinear-psd}
Let \( (V ,\inn{\cdot,\cdot}_{V}) \) be a real inner product space, and let \( \{v_1, \dots, v_m\} \subset V \) be a spanning set (not necessarily linearly independent) and let $G$ be the Gram matrix. 
Let \( A \in \mathbb{R}^{m \times m} \), and define the bilinear form:
\[
\widetilde{A} := \sum_{i,j=1}^m a_{ij} v_i \otimes v_j \in V \otimes V.
\]

Then \( \widetilde{A} \) is positive semidefinite, i.e.,
\[
\langle \widetilde{A}, u \otimes u \rangle \ge 0 \quad \text{for all } u \in \operatorname{span}\{v_1, \dots, v_m\},
\]
if and only if
\[
G A G \succeq 0.
\]
\end{cor}

Corollary \ref{cor:bilinear-psd} provides a condition for enforcing positive semidefiniteness in terms of the Gram matrix. However, to design a learning process, one must search over (or parametrize) all matrices that satisfy this property. The next lemma shows how such a parametrization can be achieved and will be useful for developing alternative learning procedures that avoid the Gram–Schmidt step.
\begin{lem}[Parametrization for PSD Bilinear Form]
\label{lem:quad-param}
Let \( G \in \mathbb{R}^{m \times m} \) be a symmetric positive semidefinite matrix with Moore--Penrose pseudoinverse \( G^\dagger \). Then:

\begin{itemize}
    \item[(i)] For any matrix \( L \in \mathbb{R}^{m \times r} \), define
    \[
    A := G^\dagger L L^\top G^\dagger.
    \]
    Then the matrix \( G A G \in \mathbb{R}^{m \times m} \) is symmetric and positive semidefinite:
    \[
    G A G = P_G L L^\top P_G \succeq 0,
    \]
    where \( P_G := G G^\dagger \) is the orthogonal projector onto \( \operatorname{Im}(G) \).

    \item[(ii)] Conversely, if \( A \in \mathbb{R}^{m \times m} \) is symmetric and \( G A G \succeq 0 \), then there exists a matrix \( L \in \mathbb{R}^{m \times r} \) such that
    \[
    A = G^\dagger L L^\top G^\dagger.
    \]
\end{itemize}
\end{lem}
\begin{proof}
See Section \ref{sec:proofs}.
\end{proof}

Finally, by combining Lemma \ref{lem:quad-param} and Corollary \ref{cor:bilinear-psd}, we obtain an equivalent formulation of Proposition \ref{prop:finite-euc}, expressed purely in terms of the Gram matrix and without relying on the Gram–Schmidt process.
\begin{prop}
\label{prop:finite-euc-Gram}
Let $S=\{k(\cdot,s_{1}),...,k(\cdot,s_{m})\}$, forms possibly linearly dependent and let \( G \in \mathbb{R}^{m \times m} \) denote the Gram matrix:
\[
G_{ij} := \langle k(\cdot,s_i), k(\cdot,s_j) \rangle_\mathcal{H}=k(s_{i},s_{j}).
\] Then:
\begin{itemize}
\item 
Solutions of \ref{eq:finite-representer} and \ref{eq:FT} can be represented in terms of kernel terms as follows,
\begin{align*}
	&A=\sum_{i,j=1}^{m}a_{ij}k(\cdot,s_{i})\otimes k(\cdot,s_{j}),\\
	&u=\sum_{i=1}^{m}b_{i}k(\cdot,s_{i}),
\end{align*}
where in above we identified the space of linear maps on $V$ with $V\otimes V$ using $\inn{\cdot,\cdot}_{\mathcal{H}}$. 
\item 
Moreover, there is a one-to-one correspondence between solutions of \ref{eq:finite-representer} with the solutions of the following problem in $\R^{m}$. 	
\begin{align}
\label{eq:finite-representer-prop-gram}
	\min_{A_{G}\in \mathbb{S}^{m}_{+}, u\in \R^{m}} 
	\sum_{i=1}^{n} \ell\Big( 
	& \Norm{e_{z^{i}_{1}}-u}^{2}_{A_{G}} - \Norm{e_{z^{i}_{2}}-u}^{2}_{A_{G}}, y^{i} \Big)  \nonumber + \lambda \Norm{u}_{A_{G}}^{2} \tag{EFG}
\end{align}

and one to one correspondence between the solutions of \ref{eq:FT} with the solutions of the following problem in $\R^{m}$,
\begin{align}
\label{eq:finite-representer-triplet-gram}
	\min_{A_{G}\in \mathbb{S}^{m}_{+}}\sum_{i=1}^{n}\ell(\Norm{e_{z^{i}_{1}}-e_{z^{i}_{2}}}^{2}_{A_G}-\Norm{e_{z^{i}_{1}}-e_{z^{i}_{3}}}^{2}_{A_{G}},y^{i})\tag{EFTG},
\end{align}
where $\mathbb{S}^{m}_{+}$ denotes space of symmetric positive definite matrices on $\R^{m}$ and for $u\in \R^{m}$ and $A\in\mathbb{S}^{m}_{+}$,
\begin{align*}
	\Norm{u}^{2}_{A}=u^{\top}Au,
\end{align*}
 and for $x=k(\cdot,s_{i})\in \mathcal{H}$, $e_{x}=e_{i}\in\R^{m}$ is the $i$-th coordinate vector and $A_{G}:=GAG$.
\end{itemize}
\end{prop}
\begin{proof}
		First item is same as the one in the Proposition \ref{prop:finite-euc}. Proof of second item follows from Corollary \ref{cor:bilinear-psd}. 
\end{proof}
\begin{rmk}
Note that Proposition~\ref{prop:finite-euc-Gram}, compared to Proposition~\ref{prop:finite-euc}, offers several advantages. 
First, it does not require any linear independence assumption, and both the metric and the ideal point admit explicit expressions in terms of kernel products of the data. 
\end{rmk}
\subsection{Kernelized Algorithms for Metric and Preference Learning}
\label{subsec:alg}
In this subsection, we demonstrate how Proposition \ref{prop:finite-euc} can be leveraged to develop a practical algorithm for the kernelized version of ideal point models. Note that both \eqref{eq:finite-representer-prop} and \eqref{eq:finite-representer-triplet} define optimization problems in Euclidean spaces, which can be solved using numerical optimization algorithms. Thus, given a dataset obtained from either pairwise or triplet comparisons, as described in Section \ref{sec:problem-euc}, along with a kernel function $k$, Proposition \ref{prop:finite-euc} enables us to solve \eqref{eq:finite-representer} and \eqref{eq:finite-representer-triplet} to learn the metric and/or the ideal point. Once a new sample $x$ arrives, Proposition \ref{prop:finite-euc} can be applied to compute its corresponding representation $\alpha_{x}$, which is then used to evaluate the metric or preference function under the nonlinear transformation induced by the kernel. In this sense, Proposition \ref{prop:finite-euc} also serves as the foundation for a nonlinear algorithm applicable to both \eqref{eq:finite-representer} and \eqref{eq:finite-representer-triplet}. We formalize this in Algorithm \ref{alg:ker_ideal_train} and Algorithm \ref{alg:ker_ideal_test} for the ideal point setting. Similar algorithms can be derived for the triplet setting.

\begin{rmk}
The sample set \( S \) in Algorithm \ref{alg:ker_ideal_test} varies based on the problem. If a larger set of embedded items is available but only a smaller set of pairwise comparisons is provided, preferences can be predicted using the subspace spanned by this set, as described in Proposition \ref{prop:finite-euc}. In an online setting, where raw unlabeled samples are unavailable in advance, the corresponding \( \alpha \) vector in Proposition \ref{prop:finite-euc} can be computed by projecting the collected samples onto the training subspace in \( \mathcal{H} \).
\end{rmk}
\begin{rmk}
As noted in Subsection \ref{subsec:Gram}, Lemma \ref{lem:quad-param} and Proposition \ref{prop:finite-euc-Gram} allow us to solve both \eqref{eq:finite-representer-prop} and \eqref{eq:finite-representer-triplet} numerically, purely in terms of the Gram matrix and without relying on the Gram–Schmidt process. Consequently, one can also consider modified variants of Algorithm \ref{alg:ker_ideal_train} and Algorithm \ref{alg:ker_ideal_test} for metric and preference learning. The advantage of this approach is that the resulting algorithms do not require computing explicit coefficient representations: both the metric and the ideal point can be expressed directly in terms of the full dataset. This makes the algorithms more interpretable, while avoiding the need to select a linearly independent subset or to perform orthogonalization. For training, one needs to search over all PSD matrices of the form \( G A G \), where \( G \) denotes the Gram matrix of the data. 
In practice, this can be efficiently implemented using Lemma~\ref{lem:quad-param} by setting 
\[
A = G^{\dagger} L L^{\top} G^{\dagger},
\]
where \( L \) is the model parameter.
\end{rmk}
\section{Experiments}
\label{sec:exper}
In this section we report experimental results.
\subsection{Experiments of Synthetic Data}
\label{subsec:exper-synthetic}
In this subsection, we present additional experimental results on synthetic data to better understand the advantages of our method over previous approaches. We conduct numerical experiments using a nonlinear data distribution to demonstrate the effectiveness of our approach. The results are summarized in Table \ref{tbl:main}. Below, we explain the data distribution setup, training details, different learning settings, and the experimental results.We then report and compare our experimental results with those from prior work.

\subsection{Data Distribution Setup}
The data distribution consists of two concentric circles with Gaussian noise of variance $0.4$. In this configuration, the left portion of the larger circle and the right portion of the smaller circle are labeled as $0$, while the remaining areas are labeled as 
$1$. This is illustrated in Figure \ref{fig:enter-label}, where label $0$
is shown in red and label $1$ in blue. For our experiments, we assume that the user prefers all points labeled $0$ over those labeled 
$1$. The objective is to identify the ideal point $u$ and the positive semi-definite (PSD) matrix $A$ that align with these preferences.
\begin{rmk}
Why did we select this configuration? The configuration presented above is chosen for its non-linear representation. It is intuitively clear that neither an ideal point nor the Mahalanobis metric in $\R^{2}$ can fully capture a user's preferences, for data distribution illustrated in Figure \ref{fig:enter-label}. 

To effectively address preference learning in this context, it is necessary to transform the data distribution non-linearly using a kernel and then perform preference learning in the transformed space generated by the kernel. This is precisely the focus of our framework discussed in Section \ref{sec:rkhs}.
\end{rmk}
\subsection{Training Details}
All reported results in Table \ref{tbl:main} are averaged over three runs, with the number in parentheses indicating the standard deviation. All training is performed under the following loss function $\ell(z_{1}, z_{2}, y) := \hing\left(y \cdot \left(\|z_{1} - u\|_{A}^{2} - \|z_{2} - u\|_{A}^{2}\right)\right) + \lambda \|u\|_{A}^{2}$ for a data pair $(z_{1},z_{2})$ and label $y$. The training loss is optimized using the Adam optimizer with a learning rate of $lr = 0.01$, and the PSD condition is enforced using Cholesky decomposition. The training-to-test data ratio is $0.3$. The testing error is reported using the $0/1$ loss. The kernel $k_{C}$ defined by $k_{C}(x,y):=\inn{x,y}+\Norm{x}^{2}\Norm{y}^{2}$ is used in the experiments. Please see \url{https://github.com/PeymanMorteza/Metric-Preference-Learning-RKHS} for other details.
\subsection{Comparison of Different Learning Methods}
\label{subsec:compare}
The results are reported in Table \ref{tbl:main}. The following experimental setting are considered. 
  \begin{figure}
    \centering
    \includegraphics[width=0.52\linewidth]{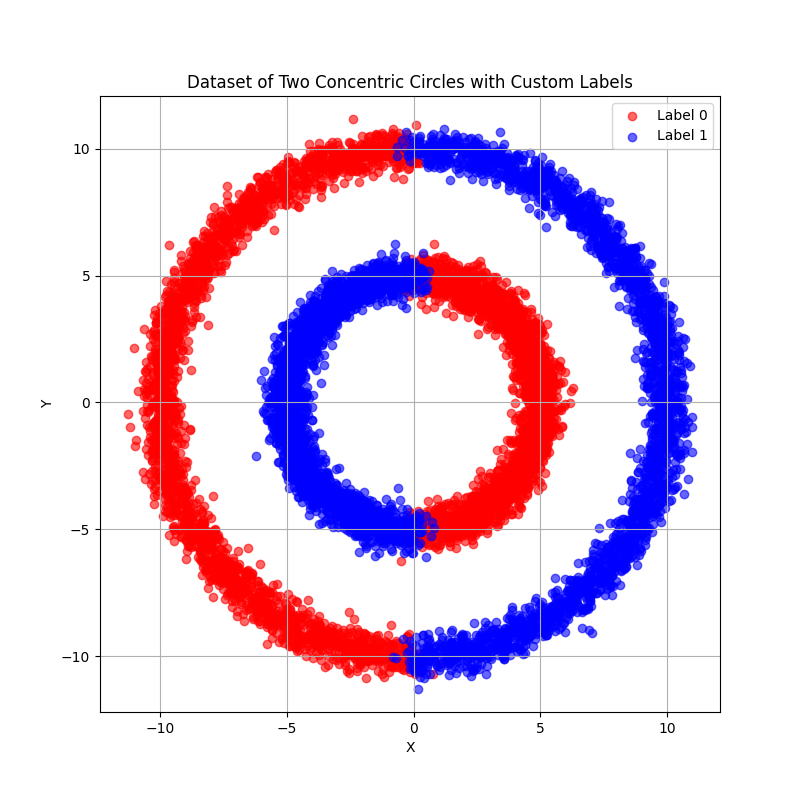}
    \caption{The data distribution is supported along two concentric circles with Gaussian noise of variance $0.4$. In this setup, the left portion of the larger circle and the right portion
of the smaller circle are labeled as $0$, while the remaining regions are labeled as $1$. This is illustrated in the figure above, with label $0$ depicted in red and label $1$ in blue. For our experiments, we assume that the user prefers all points labeled $0$ over those labeled $1$, and the objective is to search for the ideal point $u$ and the PSD matrix $A$ that align with these preferences.}
    \label{fig:enter-label}
\end{figure}
\begin{table}[h!]
    \centering
    \setlength{\abovecaptionskip}{10pt}
    \renewcommand{\arraystretch}{1.3}
    \begin{tabular}{|l|c|c|}
        \hline
        \textbf{Learning Method} & \textbf{Average Training Loss (Std Dev)} & \textbf{Testing Error(Std Dev)} \\
        \hline
        No Kernel, No Regularization & 1.07(0.04) &62.0\% (8.3)\\
        \hline
        Kernel $k_{C}$, No Regularization & 0.2 (0.3) & 17.0 \% (25)\\
        \hline
        Kernel $k_{C}$, With Regularization \bf{(ours)} & 0.25 (0.2) & 0.6\% (0.9)\\
        \hline
    \end{tabular}
    \caption{\textbf{Comparison of Different Learning Settings:} When regularizing the problem in the kernelized setting using the norm inspired by our Representer Theorem (See Theorem \ref{thm:reper} and Proposition \ref{prop:finite-euc})) , we achieve a solution with low training error that generalizes well to the test data. See Subsection \ref{subsec:compare} for further explanation.
}
\label{tbl:main}

\end{table}
\begin{itemize}
    \item \textbf{1. No Kernel, No Regularization:} This setting corresponds to what is considered in prior work (e.g. \cite{canal2022one}) in simultaneous task of metric and preference learning. Both the training loss and testing error are high. This is primarily due to the non-linear data representation, which this framework is not equipped to handle effectively.
    \item \textbf{2. Kernel $k_{C}$, No Regularization:} The kernel $k_{C}$ (defined above) is used to transform the training data, and the associated optimization problem, as given by Proposition \ref{prop:finite-euc}, is solved with no regularization (i.e., $\lambda = 0$). Small training loss with a non-trivial testing error indicates benign overfitting. This can be attributed to the lack of a regularization parameter, which plays a key role in our Representer Theorem (Theorem \ref{thm:reper}).
    \item \textbf{3. Kernel $k_{C}$, With Regularization (Ours):} The kernel $k_{C}$ (defined below) is used to transform the training data, and the associated optimization problem, as given by Proposition \ref{prop:finite-euc}, is solved with non-zero regularization(i.e., $\lambda = 0.0001$). Both the training loss and test loss are low, indicating that the method finds a solution that generalizes well in this data setting, particularly when the problem is regularized with the appropriate norm. (a
    s explained in Theorem \ref{thm:reper} and Proposition \ref{prop:finite-euc}).
\end{itemize}
 
\subsection{Experiments of Real Data}
\label{subsec:exper-real}
 \begin{table}[t]
    \centering
    \setlength{\abovecaptionskip}{10pt}
    \renewcommand{\arraystretch}{1.3}
    \begin{tabular}{|p{5.9cm}|c|c|}
        \hline
        \textbf{Algorithm} & \textbf{Chameleon} & \textbf{FlatLizard} \\
        \hline
        BT~\cite{bradley1952rank} & 0.83±0.03 & 0.86±0.06 \\
        BT-LR~\cite{chu2005preference} & 0.71±0.03 & 0.84±0.04 \\
        BT-GP~\cite{chu2005preference} & 0.75±0.04 & 0.80±0.05 \\
        RC~\cite{negahban2012iterative} & 0.61±0.06 & 0.66±0.05 \\
        RRC~\cite{jain2020spectral} & 0.61±0.03 & 0.66±0.01 \\
        SVD~\cite{cucuringu2016simple} & 0.72±0.08 & 0.69±0.05 \\
        SVDC~\cite{chau2022spectral} & 0.65±0.06 & 0.81±0.05 \\
        SVDK~\cite{chau2022spectral} & 0.76±0.06 & 0.68±0.05 \\
        Serial~\cite{fogel2016spectral} & 0.79±0.04 & 0.70±0.05 \\
        C-Serial~\cite{chau2022spectral} & 0.80±0.03 & 0.88±0.01 \\
        CC~\cite{chau2022spectral} & 0.66±0.10 & 0.78±0.08 \\
        KCC~\cite{chau2022spectral} & 0.71±0.06 & 0.78±0.03 \\
        \hline
        {Vanilla Ideal Point}~\cite{xu2020simultaneous,canal2022one} & 0.64±0.06 & 0.70±0.06 \\
        \hline
        \textbf{Kernelized Ideal Point (Ours)} & \textbf{0.83±0.08} & \textbf{0.78±0.05} \\
        \hline
    \end{tabular}
    \caption{Performance (accuracy) comparison of various methods for rank inference. The reported results for methods other than ideal point models are courtesy of \cite{chau2022spectral}. The results for the Vanilla Ideal Point Method, inspired by \cite{xu2020simultaneous, canal2022one}, are obtained by numerically solving \ref{eq:simul-finite}.}
    \label{tab:performance_comparison}
\end{table} 
In this subsection, we report experimental results to evaluate the performance of Algorithm \ref{alg:ker_ideal_train} and Algorithm \ref{alg:ker_ideal_test} on real data. We use two data sets. {\bf Flatlizard Competition:} This dataset \cite{whiting2009flat} records contests among male flat lizards.
{\bf Cape Dwarf Chameleons Contest:} This dataset \cite{stuart2006multiple} documents contests among male chameleons with associated physical measurements.

\paragraph{Experimental Setting and Results}

We split the dataset into a 70/30 train-test ratio and use it for rank inference. We compare accuracy of Algorithm \ref{alg:ker_ideal_train} and Algorithm \ref{alg:ker_ideal_test}, using RBF kernel with other spectral and probabilistic ranking methods based on pairwise comparisons (see \cite{chau2022spectral,chu2005preference}). Additionally, we compare the results with the vanilla ideal point method, a variant of which has been studied in prior works (see, for instance, \cite{canal2022one,xu2020simultaneous,massimino2021you}). The results, averaged over 10 runs, are reported in Table \ref{tab:performance_comparison}.

\section{Related Works}
\label{sec:related}
\paragraph{Representer Theorems}
Representer theorems play a central role in machine learning by reducing infinite-dimensional optimization problems to finite-dimensional ones. The concept originated in approximation theory \cite{de1966splines,kimeldorf1971some,wahba1990spline} and was later extended to a wide range of learning problems \cite{scholkopf2001generalized}. Representer theorems for Mahalanobis metric learning were established in \cite{chatpatanasiri2010new}, and this framework was directly used in \cite{tatli2025metric} to analyze the triplet setting. However, as we show in this work, these formulations do not extend to the simultaneous task of metric and preference learning.

\paragraph{Ranking Algorithms}
Ranking algorithms based on pairwise comparisons can broadly be divided into two categories. The first class follows a probabilistic approach, where generative models are employed to estimate preference probabilities. A classical example is the Bradley–Terry (BT) model \cite{bradley1952rank}, which assigns a latent score to each item and models the probability of one item being preferred over another as a function of their scores. The model parameters are typically learned using maximum likelihood estimation, and numerous extensions have been developed to handle large-scale and noisy comparison data \cite{chu2005preference,chau2022spectral}.

The second class consists of spectral methods, which exploit the spectral properties of matrices constructed from pairwise comparisons; see \cite{vigna2016spectral} for a comprehensive overview.

\paragraph{Ideal Point Models}
Prior research on simultaneous metric and preference learning has explored several directions. \cite{jamieson2011active,jamieson2011low} studied the problem from an efficiency perspective, proposing active learning strategies that minimize the number of pairwise queries needed to accurately recover rankings. \cite{massimino2021you} analyzed the recovery guarantees of such models under randomized settings, focusing on the accuracy of reconstructing the underlying preference structure. \cite{xu2020simultaneous} extended the framework to cases where the underlying metric is an unknown Mahalanobis distance, while \cite{canal2022one} generalized the classical ideal point model to settings with multiple latent ideal points. Despite these advances, all of these works are confined to finite-dimensional Euclidean spaces. By contrast, our work considers a nonlinear version of the ideal point model and establishes the first representer theorem for its kernelized version in reproducing kernel Hilbert spaces (RKHSs).
\paragraph{Metric Learning}
A large body of research has focused on learning Mahalanobis metrics from data, with the main differences among approaches stemming from how supervision is defined. Broadly, metric learning methods fall into two categories: those based on pairwise similarity or dissimilarity constraints \cite{kwok2003learning}, and those based on triplet comparisons that enforce relative distance relationships \cite{schultz2003learning}. Metric learning techniques have found wide application in areas such as large-margin nearest neighbor classification \cite{weinberger2009distance,weinberger2008fast} and unsupervised learning tasks like clustering \cite{xing2002distance}. For comprehensive reviews of this literature, see \cite{kulis2013metric,bellet2013survey,suarez2021tutorial,ghojogh2022spectral} and the references therein.

\section{Conclusion}
\label{sec:con}
In this work, we proposed a mathematical framework to address a broad class of metric and preference learning problems within Hilbert spaces. Our framework provides a principled approach to deriving representer theorems that unify different tasks under a common perspective. Technically, we established the first representer theorem for the simultaneous task of metric and preference learning, and we derived a simple, self-contained representer theorem for metric learning from triplet comparisons. In the case of RKHSs, our representer theorems enabled us to reduce infinite-dimensional learning problems to finite-dimensional counterparts, making it possible to design efficient nonlinear algorithms for metric and preference learning. From a practical standpoint, we implemented and evaluated our algorithm on real-world rank inference benchmarks. The experiments confirmed that the proposed approach is highly competitive, outperforming vanilla ideal point models and achieving results on par with or better than strong baseline methods across multiple datasets. These findings highlight both the mathematical significance and the practical applicability of our framework. Looking forward, our framework opens several directions for further research. On the theoretical side, it would be interesting to extend the analysis to other forms of structured comparisons or to explore metrics beyond Mahalanobis. On the algorithmic side, integrating our methods into modern deep learning architectures could further broaden their impact.

\appendix
\section*{Appendix}
\label{sec:appendix}

In the appendix, we present background information, review standard facts (see Section \ref{sec:back}), and include detailed proofs (see Section \ref{sec:proofs}) for our theorems from the main text.
\section{Proofs of Main Results}
\label{sec:proofs}
In this section, we provide complete proof for the theorems introduced in the previous sections. 

\begin{prop*}[Proposition~\ref{prop:inner_equiv} (restated)]
Let $A\in \mathcal{F}_{\mathcal{H}}$, then,
\begin{align*}
	\inn{x,y}_{A}:=\inn{Ax,y}_{\mathcal{H}},
\end{align*}
defines an inner product on $\mathcal{H}$ that is equivalent to $\inn{\cdot,\cdot}_{\mathcal{H}}$. Conversely, let $g(\cdot,\cdot)$ be an inner product on $\mathcal{H}$ equivalent to $\inn{\cdot,\cdot}_{\mathcal{H}}$ then there exist a unique $A\in \mathcal{F}_{\mathcal{H}}$ such that $g=\inn{\cdot,\cdot}_{A}$.
\end{prop*}
\begin{proof}
First, we show for $A\in \mathcal{F}_{\mathcal{H}}$, $\inn{\cdot,\cdot}_{A}$ defines an inner product. This part easily follows from the definition. For $x\neq 0$, we have,
\begin{align*}
	\inn{x,x}_{A}=\inn{Ax,x}_{\mathcal{H}}>0,
\end{align*}	
where we used the fact that $A$ is a (strictly) positive operator. Next,
\begin{align*}	\inn{x,y}_{A}=\inn{Ax,y}_{\mathcal{H}}=\inn{x,Ay}_{\mathcal{H}}=\inn{Ay,x}_{\mathcal{H}}=\inn{y,x}_{A},
\end{align*}
where we used the fact that $A$ is self-adjoint and $\inn{\cdot,\cdot}_{\mathcal{H}}$ is an inner product. 	Linearity for each coordinate follows from the fact that $A$ is a linear operator and $\inn{\cdot,\cdot}_{\mathcal{H}}$ is an inner product.
Next, since $A$ is a bounded operator there exist $a_{1}$ such that for $x\in \mathcal{H}$,
\begin{align*}
    \inn{x,Ax}_{\mathcal{H}}\le a_{1}\inn{x,x}_{\mathcal{H}},
\end{align*}
on the other hand, since $A$ is strictly positive, there exist $a_{2}$ such that,
\begin{align*}
    a_{2}\inn{x,x}_{\mathcal{H}}\le \inn{x,Ax}_{\mathcal{H}},
\end{align*}
above two implies that $\inn{\cdot,\cdot}_{A}$ is equivalent to $\inn{\cdot,\cdot}_{\mathcal{H}}$. Therefore, $\inn{\cdot,\cdot}_A$ is an inner product on $\mathcal{H}$ that is equivalent to $\inn{\cdot,\cdot}_{\mathcal{H}}$. Conversely, let $g$ be an inner product on $\mathcal{H}$ that is equivalent to $\inn{\cdot,\cdot}_{\mathcal{H}}$, therefor there exist $c_{1},c_{2}>0$, such that for all $y\in \mathcal{H}$,
\begin{align*}
   c_{1} \inn{y,y}_{\mathcal{H}}\le g(y,y)\le c_{2}\inn{y,y}_{\mathcal{H}}.
\end{align*}
Next, for a fixed $x\in \mathcal{H}$, consider the following functional on $\mathcal{H}$,
\begin{align*}
&\phi_{x}:\mathcal{H} \to \R\\
&\phi_{x}(y)=g(x,y),
\end{align*}
$\phi_{x}$ is a linear functional on $\mathcal{H}$ because $g$ is an inner product (e.g. linear on each coordinate). Next, we show that $\phi_{x}$ is a bounded functional,
\begin{align*}
    \abs{\phi_{x}(y)}=g(x,y)\le \sqrt{g(x,x)g(y,y)} \le C\Norm{y}_{\mathcal{H}},
\end{align*}
where $C:=\sqrt{c_{2}\cdot g(x,x)}$ and we used Cauchy-Schwartz inequality and the fact that $g$ is equivalent to $\inn{\cdot,\cdot}_{\mathcal{H}}$. Therefore, by Riesz representation theorem, there exist a unique $z_{x} \in \mathcal{H}$ so that, 
\begin{align*}
    \phi_{x}(y)=\inn{z_{x},y}_{\mathcal{H}},
\end{align*}
Now define,
\begin{align*}
    &A:\mathcal{H} \to \mathcal{H},\\
    &A(x)=z_{x},
\end{align*}
It is clear from the above construction that $A$ is a linear operator. Next we show $A$ is strictly positive,
\begin{align*}
    \inn{x,Ax}_{\mathcal{H}}=\inn{x,z_{x}}_{\mathcal{H}}=\phi_{x}(x)=g(x,x)\ge c_{1}\inn{x,x}_{\mathcal{H}},
\end{align*}
To show that $A$ is self-adjoint, for $x,y\in \mathcal{H}$, we have,
\begin{align*}
    \inn{y,Ax}_{\mathcal{H}}=\inn{y,z_{x}}_{\mathcal{H}}=\phi_{x}(y)=g(x,y)=g(y,x)=\phi_{y}(x)= \inn{x,z_{y}}_{\mathcal{H}}=\inn{Ay,x}_{\mathcal{H}},
\end{align*}
where we used the fact that $g$ is symmetric. To show uniqueness, assume there are two operators $A,B\in \mathcal{F}_{\mathcal{H}}$ such that for all $x,y\in \mathcal{H}$,
\begin{align*}
  &\inn{x,Ay}_{\mathcal{H}}=\inn{x,By}_{\mathcal{H}}.
\end{align*}
Above implies that for all $x,y \in \mathcal{H}$,
\begin{align*}
  \inn{x,(A-B)y}_{\mathcal{H}}\equiv 0 \implies Ay=By \implies A=B,
\end{align*}
and we are done. 
\end{proof}

The subsequent lemma demonstrates that when $\mathcal{H}$ is finite-dimensional with a predetermined orthonormal basis, the elements of $\mathcal{F}_{\mathcal{H}}$ align with the standard Mahalanobis norm, corresponding to the representation of the operator with respect to the basis. 
\begin{lem*}[Lemma \ref{lem:matrix-repr} (restated)]
    Let $\mathcal{H}$ be a $n$-dimensional Hilbert space and $\{e_{1},...,e_{n}\}$ be an orthonormal basis for $\mathcal{H}$. Let $M$ denote the representation of $A$ with respect to this basis. Then the inner product associated with the Mahalanobis norm coming from $M$ coincides with $\inn{\cdot,\cdot}_{A}$. 
\end{lem*}
\begin{proof}
Let $v,w\in \mathcal{H}$. We have,
\begin{align*}
    v=\sum_{i=1}^{n}\inn{v,e_{i}}_{\mathcal{H}}e_{i},\\
    w=\sum_{i=1}^{n}\inn{w,e_{i}}_{\mathcal{H}}e_{i},
\end{align*}
therefore, representation of $v$, $w$ are given by,
\begin{align*}
    &V=(v_{1},...,v_{n})^{\top} \in \R^{n},\\
    &W=(w_{1},...,w_{n})^{\top} \in \R^{n},
\end{align*}
where for $1\le i\le n$, $v_{i}=\inn{v,e_{i}}_{\mathcal{H}}$ and $w_{i}=\inn{w,e_{i}}_{\mathcal{H}}$. Next, for $1\le i\le n$ and $1\le j\le n$, $M_{ij}=\inn{Ae_{i},e_{j}}_{\mathcal{H}}$. Next, we have, 
\begin{align*}
    \inn{v,w}_{A}=\inn{v,Aw}_{\mathcal{H}}&=\inn{\sum_{i=1}^{n}\inn{v,e_{i}}_{\mathcal{H}}e_{i},A(\sum_{i=1}^{n}\inn{w,e_{i}}_{\mathcal{H}}e_{i})}_{\mathcal{H}}\\
    &=\inn{\sum_{i=1}^{n}v_{i}e_{i},\sum_{i=1}^{n}w_{i}A(e_{i})}_{\mathcal{H}}\\
    &=\sum_{1\le i,j\le n} v_{i}w_{j}\inn{e_{i},Ae_{j}}_{\mathcal{H}}=V^{\top}MW.
\end{align*}
\end{proof}
\begin{deff*}[Definition \ref{def:extension} (restated)]
Let $\mathcal{H}$ be a Hilbert space and $V\subset \mathcal{H}$ be a finite-dimensional subspace. Let $A\in \mathcal{F}_{\mathcal{H}}$ and $B \in \mathcal{F}_{V}$. We say $A$ is {\it Mahalanobis extension} of $B$ ((or $B$ is a {\it Mahalanobis restrict} of $A$) if,
\begin{align*}
   \restr{ \inn{\cdot,\cdot}_{A}}{V}=\inn{\cdot,\cdot}_{B}.
\end{align*}
\end{deff*}
The following theorem completely characterizes when $A\in \mathcal{F}_{\mathcal{H}}$ on ambient space is the Mahalanobis extension of $B\in \mathcal{F}_{V}$.
\begin{thm*}[Theorem \ref{thm:mahal-char-rest} (restated)]
$A\in \mathcal{F}_{\mathcal{H}}$ is Mahalanobis extension of $B\in \mathcal{F}_{V}$ iff,
\begin{align*}
    B=PA|_{V},
\end{align*}
where,
\begin{align*}
    P:\mathcal{H}\to V\subset \mathcal{H},
\end{align*}
is projection operator with respect to $\inn{\cdot,\cdot}_{\mathcal{H}}$. Moreover, each $B\in \mathcal{F}_{V}$ has at least one Mahalanobis extension.
\end{thm*}
\begin{proof}
    Consider the following orthogonal decomposition with respect to $\inn{\cdot,\cdot}_{\mathcal{H}}$,
\begin{align*}
	\mathcal{H}=V \oplus V^{\perp}.
\end{align*}

First, consider $A\in \mathcal{F}_{\mathcal{H}}$. Consider the projection operator with respect to $\inn{\cdot,\cdot}_{\mathcal{H}}$,
\begin{align*}
	P:\mathcal{H} \to V,
\end{align*}
and set,
\begin{align*}
	&B:V \to V\\
	&B(x)=PA(x).
\end{align*}
It is clear from the definition that $B$ is bounded and linear. Next, for $x\in V$, we also have,
\begin{align*}	\inn{Bx,x}_{\mathcal{H}}=\inn{PAx,x}_{\mathcal{H}}=\inn{Ax,x}_{\mathcal{H}}>c\Norm{x}^{2}_{\mathcal{H}},
\end{align*}
which shows that $B$ is positive. Next, for $x,y\in V$,
\begin{align*}
	\inn{Bx,y}_{\mathcal{H}}=\inn{PAx,y}_{\mathcal{H}}=\inn{Ax,y}_{\mathcal{H}}=\inn{x,Ay}_{\mathcal{H}}=\inn{x,PAy}_{\mathcal{H}},
\end{align*}
where we used the fact that $x,y\in V$ and $A$ is self-adjoint. This shows that $B$ is self-adjoint. Next, for $x,y\in V$, we have,
\begin{align*}
	\inn{x,y}_{A}=\inn{Ax,y}_{\mathcal{H}}=\inn{PAx+(Ax)^{\perp},y}_{\mathcal{H}}=\inn{PAx,y}_{\mathcal{H}}=\inn{x,y}_{B},
\end{align*}
therefore we showed that  $A$ is Mahalanobis extension of $B=PA|_{V}$. Conversely, assume $A\in \mathcal{F}_{\mathcal{H}}$ is a Mahalanobis extension of $B$ then setting $C=PA|_{V}$ we know that,
\begin{align*}
    \inn{\cdot,\cdot}_{C}=\inn{\cdot,\cdot}_{B},
\end{align*}
and it follows from Proposition \ref{prop:inner_equiv} that $B=C$. Finally, for $B\in \mathcal{F}_{V}$, define $A:=B\oplus \id$ by,
\begin{align*}
	&A:\mathcal{H} \to \mathcal{H}\\
	&A (h^{T}+h^{\perp}):=B(h^{T})+h^{\perp},
\end{align*}
where $h=h^{T}+h^{\perp}$ for $h^{T}\in V$ and $h^{\perp}\in V^{\perp}$. It can be easily checked that $A$ is positive, self-adjoint, and bounded (See Lemma \ref{lem:self-id}). It also follows from the definition of $A$ that,
\begin{align*}
	\restr{A}{V}=B,
\end{align*}
therefore $B$ has at least one Mahalanobis extension and we are done.

\end{proof}
\begin{lem}
\label{lem:self-id}
For $A\in \mathcal{F}_{V}$, define $B:=A\oplus \id$ by,
\begin{align*}
	&B:\mathcal{H} \to \mathcal{H}\\
	&B (h^{T}+h^{\perp}):=A(h^{T})+h^{\perp}.
\end{align*}
where the following orthogonal decomposition with respect to $\inn{\cdot,\cdot}_{\mathcal{H}}$ is considered,
\begin{align*}
	\mathcal{H}=V \oplus V^{\perp},
\end{align*}
and $h\in \mathcal{H}$ is represented as $h=h^{T}+h^{\perp}$. Then $B$ is positive, self-adjoint, and bounded.
\end{lem}
\begin{proof}
It is clear from the definition that $B$ is a linear operator. To show that $B$ is positive,
\begin{align*}
	\inn{Bx,x}_{\mathcal{H}}&=\inn{A(x^{T})+x^{\perp},x^{T}+x^{\perp}}_{\mathcal{H}}\\&=\inn{A(x^{T}),x^{T}}_{\mathcal{H}}+\inn{x^{\perp},x^{\perp}}_{\mathcal{H}} > c\Norm{x^{T}}^{2}_{\mathcal{H}}+\Norm{x^{\perp}}^{2}_{\mathcal{H}}\ge\min(1,c) \Norm{x}^{2}_{\mathcal{H}}>0,
\end{align*}	
where we used the fact that $A$ is a positive operator (with constant $c$) and applied the Pythagorean theorem. To show that $B$ is bounded we have,
\begin{align*}
	\inn{Bx,Bx}_{\mathcal{H}}=\inn{A(x^{T})+x^{\perp},A(x^{T})+x^{\perp}}_{\mathcal{H}}=\inn{A(x^{T}),A(x^{T})}_{\mathcal{H}}+\inn{x^{\perp},x^{\perp}}_{\mathcal{H}},
\end{align*}
and boundedness follows from the fact that $A$ is bounded. To show that $B$ is self-adjoint, we have,
\begin{align*}
	\inn{Bx,y}_{\mathcal{H}}=\inn{A(x^{T})+x^{\perp},y}_{\mathcal{H}}&=\inn{x^{\perp},y}_{\mathcal{H}}+\inn{A(x^{T}),y}_{\mathcal{H}}\\
&=\inn{x^{\perp},y^{\perp}}_{\mathcal{H}}+\inn{x^{T},A(y^{T})}_{\mathcal{H}}	\\
&=\inn{x,y^{\perp}}_{\mathcal{H}}+\inn{x,A(y^{T})}_{\mathcal{H}}	\\
&=\inn{x,By}_{\mathcal{H}}.
\end{align*}
Thus, $B$ is positive, self-adjoint, and bounded.
\end{proof}

\begin{prop*}[Proposition \ref{prop:induced-1} (restated)]
Let $A^{\ast}\in \mathcal{F}_{\mathcal{H}},u^{\ast}\in \mathcal{\mathcal{H}}$ be a solution to \ref{eq:simul-infinite}. Equip $\mathcal{H}$ with $\inn{\cdot,\cdot}_{A^{\ast}}$ and let $u^{\ast}=u^{\perp}+u^{T}$ where $u^{\perp}\in V^{\perp}$ and $u^{T} \in V$. Note that the orthogonal decomposition is with respect to $\inn{\cdot,\cdot}_{A^{\ast}}$ and not $\inn{\cdot,\cdot}_{\mathcal{H}}$. Then $(A^{\ast},u^{T})$ is also a solution to \ref{eq:simul-infinite}. Moreover, let $B^{\ast}\in \mathcal{F}_{V}$ be the Mahalanobis restrict (See Definition \ref{def:extension}) of $A^{\ast}$, then $(B^{\ast},u^{T})$ is a solution to \ref{eq:simul-finite-hilbert} with same optimal value. Conversely, let $(B^{\ast},u)$ be a solution to \ref{eq:simul-finite-hilbert}, and let $A^{\ast}\in \mathcal{F}_{\mathcal{H}}$ be any Mahalanobis extension of $B^{\ast}$ ((See Definition \ref{def:extension}) then $(A^{\ast},u)$ is also a solution to \ref{eq:simul-infinite} with same optimal value.
\end{prop*}
\begin{proof}

First note that by Proposition \ref{prop:inner_equiv}, $\inn{\cdot,\cdot}_{A^{\ast}}$ is equivalent to $\inn{\cdot,\cdot}_{\mathcal{H}}$ therefore $\mathcal{H}$ equipped with $\inn{\cdot,\cdot}_{A^{\ast}}$ is a Hilbert space. Next, decompose $u^{\ast}$ into $u^{T}\in V$ and $u^{\perp}\in V^{\perp}$ where the orthogonal decomposition in with respect to $\inn{\cdot,\cdot}_{A^{\ast}}$ and write,
	\begin{align*}
		u^{\ast}=u^{T}+u^{\perp},
	\end{align*}
For $1\le i\le n$, we have,
\begin{align*}
&\ell(\Norm{z^{i}_{1}-u^{\ast}}^{2}_{A^{\ast}}-\Norm{z^{i}_{2}-u^{\ast}}^{2}_{A^{\ast}},y^{i})\\
=&\ell(\Norm{z^{i}_{1}-(u^{T}+u^{\perp})}^{2}_{A^{\ast}}-\Norm{z^{i}_{2}-(u^{T}+ u^{\perp})}^{2}_{A^{\ast}},y^{i})\\
=&\ell(\Norm{(z^{i}_{1}-u^{T})- u^{\perp}}^{2}_{A^{\ast}}-\Norm{(z^{i}_{2}-u^{T})- u^{\perp})}^{2}_{A^{\ast}}),y^{i})\\
=&\ell(\Norm{(z^{i}_{1}-u^{T})}^{2}_{A^{\ast}}+\Norm{u^{\perp}}^{2}_{A^{\ast}}-\Norm{(z^{i}_{2}-u^{T})}^{2}_{A^{\ast}}-\Norm{ u^{\perp}}^{2}_{A^{\ast}},y^{i})\\
=&\ell(\Norm{(z^{i}_{1}-u^{T})}^{2}_{A^{\ast}}-\Norm{(z^{i}_{2}-u^{T})}^{2}_{A^{\ast}},y^{i}).
\end{align*}
Above implies that,
\begin{align*}
	\frac{1}{n}\sum_{i=1}^{n}\ell(\Norm{z^{i}_{1}-u^{\ast}}^{2}_{A^{\ast}}-\Norm{z^{i}_{2}-u^{\ast}}^{2}_{A^{\ast}},y^{i})=\frac{1}{n}	\sum_{i=1}^{n}\ell(\Norm{z^{i}_{1}-u^{T}}^{2}_{A^{\ast}}-\Norm{z^{i}_{2}-u^{T}}^{2}_{A^{\ast}},y^{i}).
\end{align*}
Therefore if $(A^{\ast},u^{\ast})$ solves \ref{eq:simul-infinite} then $(A^{\ast},u^{T})$ also solves \ref{eq:simul-infinite}. Now by Proposition \ref{prop:inner_equiv}, it is clear that $(B^{\ast},u^{T})$ solves \ref{eq:simul-finite-hilbert}. The converse follows similarly. Let $(B^{\ast},v)$, $v\in V$ be a solution to \ref{eq:simul-finite-hilbert}. Let $A^{\ast}\in \mathcal{F}_{\mathcal{H}}$ be any Mahalanobis extension. We claim that $(A^{\ast},v)$ also solves $\ref{eq:simul-infinite}$. If not, there exist $(A_{1}^{\ast},u_{1})$ with smaller loss. Then arguing as above we know that the loss for $(A_{1}^{\ast},u_{1})$ would be same as the loss for $(A_{1}^{\ast},u^{T}_{1})$ which would be same as the loss for $(B_{1}^{\ast},u^{T}_{1})$. Therefore $(B_{1}^{\ast},u^{T}_{1})$ has smaller loss that $(B^{\ast},v)$ which is a contradiction. 
\end{proof}
\begin{thm*}[Theorem \ref{thm:reper} (restated)]
Let $\lambda>0$ and consider the following infinite dimensional regularized problem,
\begin{align}
\label{eq:infinite-representer-app}
	\min_{A\in \mathcal{F}_{\mathcal{H}},u\in \mathcal{H}}\frac{1}{n}\sum_{i=1}^{n}\ell(\Norm{z^{i}_{1}-u}^{2}_{A}-\Norm{z^{i}_{2}-u}^{2}_{A},y^{i})+\lambda \Norm{u}^{2}_{A}, \tag{R}
\end{align}
and it's finite dimensional equivalent as follows,
\begin{align}
\label{eq:finite-representer-app}
	\min_{A\in \mathcal{F}_{V},u\in V}\frac{1}{n}\sum_{i=1}^{n}\ell(\Norm{z^{i}_{1}-u}^{2}_{A}-\Norm{z^{i}_{2}-u}^{2}_{A},y^{i})+\lambda \Norm{u}^{2}_{A}\tag{F},
\end{align}
Let $A^{\ast}\in \mathcal{F}_{H}$ be any Mahalanobis extension of $B^{\ast}\in \mathcal{F}_{V}$. Then $(A^{\ast},u)$ is a solution to  \ref{eq:infinite-representer-app} iff $(B^{\ast},u)$ is a solution to \ref{eq:finite-representer-app} with same optimal value. In particular this implies that $u\in V$. 
\end{thm*}

 \begin{proof}
 Let $(A^{\ast},u)$ be a solution to \ref{eq:infinite-representer-app}. By Proposition \ref{prop:induced-1}, we get a smaller loss for the regularized term when we project on $V$ (with respect to $\inn{\cdot,\cdot}_{A^{\ast}})$ while keeping the value of the other term unchanged, therefore $u\in V$. Next, we claim that $(B^{\ast},u)$ is a solution to \ref{eq:finite-representer-app}. If not, there exist a solution $(B_{1},v_{1})$ with smaller loss. Let $A_{1}$ be any Mahalanobis extension of $B_{1}$ obtained from  \ref{thm:mahal-char-rest}. It follows that $(A_{1},v_{1})$ has a same loss value for \ref{eq:infinite-representer-app} as $(B_{1},v_{1})$ for \ref{eq:finite-representer-app}. Therefore,  $(A_{1},v_{1})$ has smaller loss than $(A^{\ast},u)$ which is a contradiction. The converse follows with a similar argument and we are done. 
 \end{proof}

\begin{lem}
\label{lem:reper_euc-app}
Consider the same setting as in Theorem \ref{thm:gram-schmidt} and let $x\in \mathcal{H}$ and write  $x=x^{T}+x^{\perp}$ where the decomposition is with respect to $\inn{\cdot,\cdot}_{\mathcal{H}}$. Then $x^{T}$ can be represented by $\alpha_{x}=(\alpha_{1},...,\alpha_{m})^{\top}\in \R^{m}$ as follows, 
\begin{align*}
	x^{T}=\sum_{i=1}^{m}\inn{x,e_{i}}_{\mathcal{H}}e_{i}=\sum_{i=1}^{m}\alpha_{i}e_{i},
\end{align*}
and $\alpha_{i}$ is defined by,
\begin{align*}
	\alpha_{i}=\frac{1}{\sqrt{D_{i-1}D_{i}}}\begin{vmatrix}
\inn{x_{1},x_{1}}_{\mathcal{H}} & \dots & \dots & \inn{x_{1},x_{i}}_{\mathcal{H}} \\ 
\vdots & \ddots & \ddots & \vdots \\ 
\inn{x_{i-1},x_{1}}_{\mathcal{H}} & \dots & \dots & \inn{x_{i-1},x_{i}}_{\mathcal{H}} \\ 
\inn{x,x_{1}}_{\mathcal{H}} & \dots & \dots & \inn{x,x_{i}}_{\mathcal{H}} \notag
\end{vmatrix}.
\end{align*}
\end{lem}
\begin{proof}
This simply follows from {\it Leibniz formula for determinants} applying to $e_{i}$ obtaining from Theorem \ref{thm:gram-schmidt}. 
\end{proof}
\begin{lem*}[Lemma \ref{lem:gram-lemma} (restated)]
Let \( (V,\inn{\cdot,\cdot}_{V}) \) be a real inner product space, and let \( \{v_1, \dots, v_m\} \subset V \) be a (possibly linearly dependent) set of vectors. Let \( G \in \mathbb{R}^{m \times m} \) denote the Gram matrix:
\[
G_{ij} := \langle v_i, v_j \rangle_V.
\]
Let \( A \in \mathbb{R}^{m \times m} \), and define the bilinear form \( \widetilde{A} \in V \otimes V \) as:
\[
\widetilde{A} := \sum_{i,j=1}^m a_{ij} \, v_i \otimes v_j.
\]
Then for any \( u=(u_1,...,u_m)^{\top} \in \mathbb{R}^m \), the vector \( \bar{u} := \sum_{k=1}^m u_k v_k \in \operatorname{span}\{v_i\} \) satisfies:
\[
\langle \widetilde{A}, \bar{u} \otimes \bar{u} \rangle = u^\top G A G u.
\]

\end{lem*}
\begin{proof}
We begin by expanding \( \bar{u} \otimes \bar{u} \):
\[
\bar{u} \otimes \bar{u} = \left( \sum_{l=1}^m u_l v_l \right) \otimes \left( \sum_{t=1}^m u_t v_t \right) = \sum_{l,t=1}^m u_l u_t \, v_l \otimes v_t.
\]

Similarly, \( \widetilde{A} = \sum_{i,j=1}^m a_{ij} \, v_i \otimes v_j \). Then using bilinearity of the inner product on \( V \otimes V \), we compute:
\[
\langle \widetilde{A}, \bar{u} \otimes \bar{u} \rangle
= \left\langle \sum_{i,j} a_{ij} v_i \otimes v_j, \sum_{l,t} u_l u_t v_l \otimes v_t \right\rangle
= \sum_{i,j,l,t} a_{ij} u_l u_t \langle v_i \otimes v_j, v_l \otimes v_t \rangle.
\]

By the definition of the inner product on the tensor space, we have:
\[
\langle v_i \otimes v_j, v_l \otimes v_t \rangle = \langle v_i, v_l \rangle \cdot \langle v_j, v_t \rangle = G_{il} G_{jt}.
\]

Substituting back, we obtain:
\[
\langle \widetilde{A}, \bar{u} \otimes \bar{u} \rangle
= \sum_{i,j,l,t} a_{ij} u_l u_t G_{i l} G_{j t}.
\]

Now define the vector \( Gu \in \mathbb{R}^m \) with components \( (Gu)_i = \sum_l G_{i l} u_l \), and observe that:
\[
\sum_{i,j} (Gu)_i \, a_{ij} \, (Gu)_j = (Gu)^\top A (Gu) = u^\top G A G u.
\]
\end{proof}
\begin{lem}
\label{lem:ImG-ImL}
Let $M = L L^{\top}$ with $L \in \mathbb{R}^{m\times r}$. Then
\[
\ker(M) \;=\; \ker(L^{\top}) \qquad\text{and}\qquad \operatorname{Im}(M) \;=\; \operatorname{Im}(L).
\]
\end{lem}

\begin{proof}
First, $\ker(M)=\ker(L^\top)$:
\[
Mx = 0 \;\Longleftrightarrow\; x^{\top}Mx = 0 \;\Longleftrightarrow\; \|L^{\top}x\|^2 = 0 \;\Longleftrightarrow\; L^{\top}x=0.
\]
Next, since $M$ is symmetric, the range–kernel orthogonality gives,
\[
\operatorname{Im}(M) \;=\; \ker(M)^{\perp} \;=\; \ker(L^{\top})^{\perp} \;=\; \operatorname{Im}(L),
\]
as claimed.
\end{proof}
\begin{lem}
\label{lem:proj-gram}
Let $G\in\mathbb{R}^{m\times m}$ be a Gram matrix.
Define $P_G \;:=\; GG^\dagger$. Then
\[
P_G:= GG^\dagger \;=\; G^\dagger G \;=\; (G^\dagger G)^\top,
\]
and $P_G$ is the orthogonal projector onto $\operatorname{Im}(G)$.
\end{lem}

\begin{proof}
Since $G$ is symmetric PSD, it admits an eigen-decomposition
$G = U \Lambda U^\top$, where $U=[u_1,..,u_m]$ is orthogonal and
$\Lambda=\mathrm{diag}(\lambda_1,\dots,\lambda_m)$ with $\lambda_i\ge 0$.
The pseudoinverse is $G^\dagger = U \Lambda^\dagger U^\top$, where
$\Lambda^\dagger=\mathrm{diag}(\lambda_1^\dagger,\dots,\lambda_m^\dagger)$ and
$\lambda_i^\dagger = 1/\lambda_i$ if $\lambda_i>0$, and $0$ otherwise. We have,
\[
GG^\dagger
= U(\Lambda \Lambda^\dagger)U^\top
= U\,\mathrm{diag}(\mathbf{1}_{\{\lambda_i>0\}})\,U^\top=G^\dagger G.
\]
Thus $GG^\dagger = G^\dagger G$ and $P_G=GG^\dagger$ is symmetric and $P_G^2=P_G$. Therefore, $P_G$
is orthogonal projector onto its image. Finally, we have, $\operatorname{Im}(P_G) = \mathrm{span}\{u_i:\lambda_i>0\} = \operatorname{Im}(G)$.

\end{proof}

\begin{lem*}[Lemma \ref{lem:quad-param} (restated)]
Let \( G \in \mathbb{R}^{m \times m} \) be a symmetric positive semidefinite matrix with Moore-Penrose pseudoinverse \( G^\dagger \). Then:

\begin{itemize}
    \item[(i)] For any matrix \( L \in \mathbb{R}^{m \times r} \), define
    \[
    A := G^\dagger L L^\top G^\dagger.
    \]
    Then the matrix \( G A G \in \mathbb{R}^{m \times m} \) is symmetric and positive semidefinite:
    \[
    G A G = P_G L L^\top P_G \succeq 0,
    \]
    where \( P_G := G G^\dagger \) is the orthogonal projector onto \( \operatorname{Im}(G) \).

    \item[(ii)] Conversely, if \( A \in \mathbb{R}^{m \times m} \) is symmetric and \( G A G \succeq 0 \), then there exists a matrix \( L \in \mathbb{R}^{m \times r} \) such that
    \[
    A = G^\dagger L L^\top G^\dagger.
    \]
\end{itemize}
\end{lem*}

\begin{proof}
(i) Let \( A = G^\dagger L L^\top G^\dagger \). Then:
\[
G A G = G G^\dagger L L^\top G^\dagger G = P_G L L^\top P_G.
\]
Since \( P_G \) is an orthogonal projector (See Lemma \ref{lem:proj-gram}) and \( L L^\top \succeq 0 \), we conclude that \( G A G \succeq 0 \).

(ii) Suppose \( A \) is symmetric and \( G A G \succeq 0 \). Then \( G A G \) admits a factorization:
\[
G A G = \tilde{L} \tilde{L}^\top.
\]
Since \( \operatorname{Im}(\tilde{L})=\operatorname{Im}(G A G) \subseteq \operatorname{Im}(G) =  \) (See Lemma \ref{lem:ImG-ImL}), we can write \( \tilde{L} = P_G L \) for some matrix \( L \), and define:
\[
A := G^\dagger L L^\top G^\dagger.
\]
Then:
\[
G A G = G G^\dagger L L^\top G^\dagger G = P_G L L^\top P_G = \tilde{L} \tilde{L}^\top,
\]
as desired.
\end{proof}
\section{Background}
\label{sec:back}
Here we review some math background for completeness. We refer to standard references for a more detailed discussion.
\subsection{Leibniz Formula}
\begin{thm}
\label{thm:lebniz}
Let $A$ be a $n \times n$ square matrix. Then we have the following for its determinant,
\begin{align*}
    \det(A)=\sum_{\tau \in S_{n}}\sgn(\tau)\prod_{i=1}^{n}a_{i\tau(i)}
\end{align*},
where $a_{ij}$ denotes the $i$ and $j$ entries and $S_{n}$ denotes the symmetric group on $n$ letters and for $\tau \in S_{n}$, $\sgn(\tau)$ denotes the sign of the permutation.
\end{thm}
\begin{rmk}
When computing the formal determinant, as exemplified in Theorem \ref{thm:gram-schmidt}, we calculate a formal determinant where the last row consists of vectors and the other rows consist of scalars. In this scenario, the determinant can be interpreted using the Leibniz formula mentioned earlier, where for $\tau \in S_{n}$, $\sgn(\tau)\prod_{i=1}^{n}a_{i\tau(i)}$ represents the product of $n-1$ scalars and one vector. Consequently, the resultant outcome can be viewed as a vector.
\end{rmk}
\subsection{Tensor Product and Musical Isomorphism}

Here, we review some properties of the tensor product and the dual of a vector space, as utilized in the statement of Proposition \ref{prop:finite-euc}. These properties are well-documented in standard references; see, for example, \cite{lang2012algebra}. Let $V$ and $W$ be two vector spaces. Their tensor product $V \otimes W$ is a vector space consisting of all formal sums:

\[
V \otimes W = \left\{\sum_{i=1}^{n} v_{i} \otimes w_{i} \mid v_{i} \in V, w_{i} \in W, n \in \mathbb{N}\right\},
\]

such that linearity is preserved in each coordinate. Specifically, for $a \in \mathbb{R}$:

\[
a \cdot \sum_{i=1}^{n} v_{i} \otimes w_{i} = \sum_{i=1}^{n} (a \cdot v_{i}) \otimes w_{i} = \sum_{i=1}^{n} v_{i} \otimes (a \cdot w_{i}).
\]

 When $V$ is $n$-dimensional with a basis $\{v_{1}, \ldots, v_{n}\}$, and $W$ is $m$-dimensional with a basis $\{w_{1}, \ldots, w_{m}\}$, it can be shown that $V \otimes W$ is $nm$-dimensional, with a basis given by $\{v_{i} \otimes w_{j} \mid 1 \leq i \leq n, 1 \leq j \leq m\}$. 
Next, let \( V \) be a finite-dimensional vector space with an inner product \( \langle \cdot, \cdot \rangle_{g} \). The dual space \( V^{\ast} \) consists of all linear functionals on \( V \). The space of linear maps, such as \( T: V \to V \), can be represented as \( V \otimes V^{\ast} \). Moreover, \( V \) can be naturally identified with \( V^{\ast} \) via the so-called \textit{musical isomorphism}, defined as:

\[
\flat: v \in V \mapsto \langle v, \cdot \rangle_{g} \in V^{\ast},
\]
\[
\sharp: \langle v, \cdot \rangle_{g} \in V^{\ast} \mapsto v \in V.
\]

Therefore, the space of linear maps on \( V \), identified with \( V \otimes V^{\ast} \), can further be identified with \( V \otimes V \) using the musical isomorphism. This identification is utilized in the statement of Proposition \ref{prop:finite-euc}. Finally, given two inner product spaces \( (V, \langle \cdot, \cdot \rangle_{V}) \) and \( (W, \langle \cdot, \cdot \rangle_{W}) \), the tensor product \( V \otimes W \) inherits a natural inner product defined on pure tensors by
\[
\langle v_{1} \otimes w_{1},\, v_{2} \otimes w_{2} \rangle_{V \otimes W}
\;=\; \langle v_{1}, v_{2} \rangle_{V} \, \langle w_{1}, w_{2} \rangle_{W},
\]
and extended bilinearly to all of \( V \otimes W \). This construction ensures that \( V \otimes W \) is itself an inner product space, compatible with the identifications discussed above. This is used in Lemma \ref{lem:gram-lemma}.

 \subsection{Gram-Schmidt Process}
 \label{sec:gram-schmidt}
In this section, we recall the determinant form of the Gram-Schmidt process. We refer to \cite{gantmacher1959theory} for detailed discussion. Let $\mathcal{H}$, equipped with $\inn{\cdot,\cdot}_{\mathcal{H}}$ be an inner product space. 
\begin{thm}[Gram-Schmidt Determinant Formula]
\label{thm:gram-schmidt}
Let $\{x_{1},...,x_{m}\} \in \mathcal{H}$ be linearly independent and set,
\begin{align*}
	V=\spn\{x_{1},...,x_{m}\} \subset \mathcal{H}.
\end{align*}
Then $\{e_i\}_{i=1}^{m}$ forms an orthonormal basis for $V$ where for $1\le i \le m$, $e_{i}$ is defined by,
\begin{align*}
	e_{i}=\frac{1}{\sqrt{D_{i-1}D_{i}}}\begin{vmatrix}
\inn{x_{1},x_{1}}_{\mathcal{H}} & \dots & \dots & \inn{x_{1},x_{i}}_{\mathcal{H}} \\ 
\vdots & \ddots & \ddots & \vdots \\ 
\inn{x_{i-1},x_{1}}_{\mathcal{H}} & \dots & \dots & \inn{x_{i-1},x_{i}}_{\mathcal{H}} \\ 
x_{1} & \dots & \dots & x_{i}  \notag
\end{vmatrix},
\end{align*}
where for each $1\le i\le m$, we compute a {\it formal $i\times i$ determinant} and $D_{0}=1$ and for $1\le i\le m$, $D_{i}$ is defined as follows,
\begin{align*}
	D_{i}=\begin{vmatrix}
\inn{x_{1},x_{1}}_{\mathcal{H}} & \dots & \dots & \inn{x_{1},x_{i}}_{\mathcal{H}} \\ 
\vdots & \ddots & \ddots & \vdots \\ 
\inn{x_{i-1},x_{1}}_{\mathcal{H}} & \dots & \dots & \inn{x_{i-1},x_{i}}_{\mathcal{H}} \\ 
\inn{x_{i},x_{1}}_{\mathcal{H}} & \dots & \dots & \inn{x_{i},x_{i}}_{\mathcal{H}}  \notag
\end{vmatrix}.
\end{align*}
\end{thm}
\begin{rmk}
When computing the formal determinant, as exemplified in Theorem \ref{thm:gram-schmidt}, we calculate a formal determinant where the last row consists of vectors and the other rows consist of scalars. In this scenario, the determinant can be interpreted using the Leibniz formula mentioned earlier, where for $\tau \in S_{n}$, $\sgn(\tau)\prod_{i=1}^{n}a_{i\tau(i)}$ represents the product of $n-1$ scalars and one vector. Consequently, the resultant outcome can be viewed as a vector.
\end{rmk}
\subsection{Linear Operators on Hilbert space}
Here, we review some standard facts about bounded operators on Hilbert spaces. For a more detailed discussion, see \cite{folland1999real}.
Throughout this subsection, we assume $\mathcal{H}$ is a (real) Hilbert space equipped with inner product $\inn{\cdot,\cdot}_{\mathcal{H}}$ and associated norm $\Norm{\cdot}_{\mathcal{H}}$.
\begin{deff}
A linear operator,
\begin{align*}
    A:\mathcal{H}\to \mathcal{H},
\end{align*}
is called bounded if one of the following equivalent conditions holds,
\begin{itemize}
    \item $A$ is continuous at $0\in \mathcal{H}$,
    \item $A$ is continuous,
    \item There exist $c>0$ such that $\Norm{Ax}_{\mathcal{H}}\le c$ for all $x\in \mathcal{H}$ with $\Norm{x}_{\mathcal{H}}\le 1$,
    \item There exist $c>0$ such that $\Norm{Ax}_{\mathcal{H}} \le c \Norm{x}_{\mathcal{H}}$ for all $x\in \mathcal{H}$,
\end{itemize}
\end{deff}
\begin{rmk}
When $\mathcal{H}$ is $n$-dimensional (e.g., $\mathbb{R}^{n}$ equipped with the standard inner product), the linear map $A$ can be represented by an $n \times n$ matrix $M$ which depends on the basis chosen for $\mathcal{H}$. Therefore we can think of bounded linear operators as natural generalizations of matrices.
\end{rmk}
\begin{thm}[Riesz Representation Theorem]
    Let $\phi:\mathcal{H}\to \R$ be a bounded functional. Then there exist a unique $z_{\phi}\in \mathcal{H}$, denoted by Riesz representation of $\phi$, such that for all $x\in \mathcal{H}$,
    \begin{align*}
        \phi(x)=\inn{x,z_{\phi}}_{\mathcal{H}}.
    \end{align*}
\end{thm}
\begin{deff}
Let 
\begin{align*}
    A:\mathcal{H}\to \mathcal{H},
\end{align*}
be a bounded linear operator. The adjoint of $A$, denoted by $A^{\ast}$ is bounded linear operator on $\mathcal{H}$ that for any $x,y \in\mathcal{H}$ satisfies,
\begin{align*}
    \inn{Ax,y}=\inn{x,A^{\ast}y}.
\end{align*}
An operator $A$ is called self-adjoint if $A=A^{\ast}$
\end{deff}
\begin{rmk}
The existence and uniqueness of the adjoint follow from the Riesz representation theorem. In the case where $\mathcal{H}$ is $n$-dimensional (e.g., $\mathbb{R}^{n}$ equipped with the standard inner product), and the linear map $A$ is represented by an $n \times n$ matrix $M$, then the adjoint of $A$ corresponds to its transpose, denoted as $M^{\top}$. Therefore, self-adjoint operators on Hilbert spaces can be considered as a generalization of symmetric matrices.
\end{rmk}
\begin{deff}
A bounded linear operator on $\mathcal{H}$ is called strictly positive (or bounded from below) if there exist $c>0$ such that for any $x\in \mathcal{H}$,
\begin{align*}
    \inn{x,Ax}_{\mathcal{H}}>c\cdot \Norm{x}^{2}_{\mathcal{H}}.
\end{align*}
\end{deff}
\begin{rmk}
In the case where $\mathcal{H}$ is $n$-dimensional (e.g., $\mathbb{R}^{n}$ equipped with the standard inner product), and the linear map $A$ is represented by an $n \times n$ matrix $M$, the positivity of $A$ translates to $M$ being a positive definite matrix. Consequently, positive operators on Hilbert spaces can be regarded as a generalization of positive definite matrices.
\end{rmk}
We next recall the notion norm equivalence of vector spaces,
\begin{deff}
    Let $g_{1},g_{2}$ be two inner product on a Hilbert space $\mathcal{H}$. We say $g_{1}$ and $g_{2}$ are equivalent if there exist constant $c_{1}$ and $c_{2}$ such that for all $x\in \mathcal{H}$,
\begin{align*}
    c_{1}g_{1}(x,x)\le g_{2}(x,x)\le c_{2}g_{1}(x,x).
\end{align*}
    \end{deff}
We next recall the definition of space of generalized Mahalanobis inner product from main body,
\begin{deff}[Space of generalized Mahalanobis inner products]
Space of generalized Mahalanobis inner product on a Hilbert space $\mathcal{H}$ is defined the following set,
\begin{align*}
	\mathcal{F}_{\mathcal{H}}:=\{A:\mathcal{H}\to \mathcal{H}|A \text{ is bounded, strictly positive, and self-adjoint\}}.
\end{align*}
\end{deff}

In light of the above discussion, we can regard elements of $\mathcal{F}_{\mathcal{H}}$ as a generalized version of symmetric, positive-definite matrices in finite-dimensional space.

\subsection{Kernels and RKHS}
In this subsection, we review some standard concepts related to kernels and reproducing kernel Hilbert spaces that is used in the main body of the paper. These can be found in standard references such as \cite{smola1998learning,song2009hilbert}. Throughout this discussion, let $\mathcal{X}$ denote the feature space.
\begin{deff}
A (real) kernel on $\mathcal{X}$ is a mapping $k:\mathcal{X}\times \mathcal{X}\to \R$, such that, it is symmetric, i.e. for any $x,y\in \mathcal{X}$, $k(x,y)=k(y,x)$, and for any $\{x_{1},...,x_{n}\}\subset \mathcal{X}$ the corresponding $n\times n$ Gram matrix $K$ defined by, $K_{ij}=k(x_{i},x_{j})$, is positive definite.
\end{deff}

Given a kernel function $k$ is given one can construct a special Hilbert space associated to it called universal RKHS. First, consider the following vector space,$\mathcal{H}:=\{f|f:\mathcal{X}\to \R\}$, and consider the feature map,
\begin{align*}
&\Phi:\mathcal{X}\to \mathcal{H}\\
&\Phi(x):=k(x,.).	
\end{align*}
Next, let $\mathcal{H}_{k}$ denotes the vector space generated by $\Phi(\mathcal{X})$. More formally,  $\mathcal{H}_{k}:=\{\sum_{i=1}^{m}a_{i}k(.,x_{i})|a_{i},x_{i}\in\R, m\in \N \}\subset \mathcal{H}$. One can define the inner product on $\mathcal{H}_{k}$ first by, $\inn{k(x,.),k(y,.)}:=k(x,y)$,
and extend above to $\mathcal{H}_{k}$ by linearity. Notice that above implies that for $f\in \mathcal{H}_{k}$, $f(x)=\inn{f,k(x,.)}$,
which is referred to as reproducing property. The completion of $\mathcal{H}_{k}$ (also denoted by $\mathcal{H}_{k}$) is called the \emph{reproducing kernel Hilbert space} (RKHS) associated with the kernel $k$. It can be shown that each kernel function uniquely determines an RKHS, whose elements lie in $\mathcal{H}$, as stated in the following theorem.
\begin{thm}[\cite{aronszajn1950theory}]
    Let $k$ be a symmetric positive definite kernel. Then there exist a unique Hilbert space of functions for which $k$ is reproducing kernel. 
\end{thm}
\subsection{Moore–Penrose Pseudoinverse}
In this subsection, we recall the definition and basic properties of the Moore–Penrose pseudoinverse \cite{penrose1955generalized}. For a detailed discussion, we refer the reader to \cite{ben2003generalized}. This concept is used in Subsection \ref{subsec:Gram} to eliminate the need for the Gram–Schmidt process.
\begin{deff}[\cite{penrose1955generalized}]
Let \( M \in \mathbb{R}^{m \times n} \) be a real matrix. The \textit{Moore--Penrose pseudoinverse} of \( M \), denoted \( M^\dagger \in \mathbb{R}^{n \times m} \), is the unique matrix satisfying the following four properties:
\begin{align*}
(1) \quad & M M^\dagger M = M \\
(2) \quad & M^\dagger M M^\dagger = M^\dagger \\
(3) \quad & (M M^\dagger)^\top = M M^\dagger \\
(4) \quad & (M^\dagger M)^\top = M^\dagger M
\end{align*}

If \( M \) is full-rank and square, then \( M^\dagger = M^{-1} \). If \( M \) is symmetric and positive semidefinite (as in Gram matrices), then the pseudoinverse can be computed via eigen-decomposition:
\[
M = U \Lambda U^\top \quad \Rightarrow \quad M^\dagger = U \Lambda^\dagger U^\top,
\]
where \( \Lambda^\dagger \) is obtained by taking reciprocal of nonzero eigenvalues and leaving zeros unchanged.
\end{deff}
\section*{Acknowledgments} 
I would like to thank the anonymous reviewers for their many helpful suggestions, including comments related to  \cite{chau2022spectral} and its associated baselines and datasets, which provided a valuable setting for evaluating the proposed algorithms.
I also thank Rob Nowak and Ramya Vinayak for helpful discussions and for bringing this general line of work (\cite{chatpatanasiri2010new}, \cite{tatli2025metric}) to my attention.
I further thank Remzi Arpaci-Dusseau for helpful conversations.
Experiments reported in this work were partly supported by the Google Cloud Research Credits program under award GCP19980904.

\bibliography{Peyman_note}
\addcontentsline{toc}{section}{Acknowledgments}
\addcontentsline{toc}{section}{References}
\end{document}